%% file: graphs_arxiv.tex
\documentclass[11pt]{article}
\usepackage[colorlinks,allcolors=blue]{hyperref}
\usepackage{amsmath,amssymb,amsthm,subfig,url,mathtools}
\usepackage[ruled,vlined]{algorithm2e}
\usepackage[margin=1in]{geometry}
\usepackage[capitalize]{cleveref}
\usepackage[round]{natbib}
\usepackage[extdef=true]{delimset}
\usepackage{xcolor}
\usepackage{float}

\usepackage[font=small]{caption}

\usepackage{selectp}

\newtheorem{manualrestateinner}{}
\newenvironment{manualrestate}[1]{%
  \renewcommand\themanualrestateinner{#1}%
  \manualrestateinner
}{\endmanualrestateinner}

\allowdisplaybreaks

\newcommand{\ifrac}[2]{\mathchoice{\frac{#1}{#2}}{#1/#2}{#1/#2}{#1/#2}}

\newcommand{\diag}{\operatorname{diag}}
\newcommand{\tr}{^{\mathsf{T}}}
\newcommand{\argmin}{\operatorname{arg\,min}}

\newcommand{\Otilde}{\widetilde{\mathcal{O}}}
\newcommand{\dotp}{\boldsymbol{\cdot}}
\newcommand{\R}{\mathbb{R}}
\newcommand{\E}{\mathbb{E}}
\newcommand{\eqdef}{\triangleq}

\newcommand{\simplex}{\mathcal{S}}

\newcommand{\grad}{\nabla}
\newcommand{\W}{\mathcal{W}}

\newcommand{\regret}{\mathcal{R}_T}
\newcommand{\phalf}{p_t^+}
\newcommand{\phalfi}[1]{p^+_{t,{#1}}}
\newcommand{\phalfclique}[1]{p_t^+(\clique_{#1})}

\newcommand{\phalfcliqueminus}[2]{p_t^+ \brk*{\clique_{#1} \!\setminus\! {#2}}}
\newcommand{\whalf}{w_t^+}
\newcommand{\Fhalf}{F_t^+}
\newcommand{\simplexftrl}{\mathcal{S}^\gamma_N}
\newcommand{\corell}{\ell}
\newcommand{\estell}{\smash{\widehat{\ell}}}
\newcommand{\iidell}{\smash{\tilde{\ell}}}
\newcommand{\clique}{V}
\newcommand{\covnum}{\theta}
\newcommand{\pclique}[1]{p(\clique_{#1})}
\newcommand{\ptclique}[2]{p_{#1}(\clique_{#2})}
\newcommand{\ptcliquei}[2]{p_{#1} \brk*{\clique({#2})}}
\newcommand{\Deltaarm}[1]{\delta_{#1}}
\newcommand{\Deltaclique}[1]{\Delta_{#1}}
\newcommand{\neighbors}{\mathcal{N}}

\newcommand{\sumclique}{\mathchoice
    {\sum_{\mathclap{k : \Deltaclique{k} > 0}}\;} 
    {\sum_{k : \Deltaclique{k} > 0}}
    {\sum_{k : \Deltaclique{k} > 0}}
    {\sum_{k : \Deltaclique{k} > 0}}
}

\newcommand{\sumcliqueneq}{\mathchoice
    {\sum_{\mathclap{k \neq k^\star}}} 
    {\sum_{k \neq k^\star}}
    {\sum_{k \neq k^\star}}
    {\sum_{k \neq k^\star}}
}

\newcommand{\deltasum}[1]{\sumclique\ifrac{#1}{\Deltaclique{k}}}

\newtheorem{theorem}{Theorem}
\newtheorem{lemma}{Lemma}

\newtheorem*{theorem*}{Theorem}

\title{Best-of-All-Worlds Bounds for\\ Online Learning with Feedback Graphs}
\author{%
Liad Erez\footnote{
Blavatnik School of Computer Science,
Tel Aviv University;
\texttt{liaderez1@gmail.com}.}
\and
Tomer Koren\footnote{
Blavatnik School of Computer Science,
Tel Aviv University, and Google Research;
\texttt{tkoren@tauex.tau.ac.il}.}
}

\begin{document}
\maketitle

\input{graphs_content}

\end{document}

%% file: graphs_content.tex
\begin{abstract}
    We study the online learning with feedback graphs framework introduced by \citet{mannor2011bandits}, in which the feedback received by the online learner is specified by a graph $G$ over the available actions. 
    We develop an algorithm that simultaneously achieves regret bounds of the form: 
    $\smash{\mathcal{O}(\sqrt{\covnum(G) T})}$ with adversarial losses;
    $\mathcal{O}(\covnum(G)\operatorname{polylog}{T})$ with stochastic losses; 
    and $\mathcal{O}(\covnum(G)\operatorname{polylog}{T} + \smash{\sqrt{\covnum(G) C})}$ with stochastic losses subject to $C$ adversarial corruptions.
    Here, $\covnum(G)$ is the \textit{clique covering number} of the graph $G$.
    Our algorithm is an instantiation of Follow-the-Regularized-Leader with a novel regularization that can be seen as a product of a Tsallis entropy component (inspired by \citet{zimmert2019optimal}) and a Shannon entropy component (analyzed in the corrupted stochastic case by \citet{amir2020prediction}), thus subtly interpolating between the two forms of entropies.
    One of our key technical contributions is in establishing the convexity of this regularizer and controlling its inverse Hessian, despite its complex product structure.
\end{abstract}

\section{Introduction}


Online learning models a repeated interaction between a learner and an environment. In this framework, the learner has to choose an arm (or action) from a set of $N$ arms, iteratively across $T$ rounds. After choosing an arm the learner incurs an associated loss and receives feedback from the environment. 
The feedback the learner receives can range between the \emph{full-information} setting where after choosing an action the learner sees the losses of all $N$ arms, and the \emph{multi-armed bandit} (MAB) setting where only the loss of the arm chosen at the current round is revealed. 
Traditionally, the way the losses are generated is either in an adversarial manner (i.e., can be arbitrary in each round), or stochastically (i.e., sampled i.i.d.~from an unknown distribution).

Recently, online algorithms achieving best-of-both-worlds guarantees have drawn
significant attention: these are algorithms that are able to perform optimally
in both the stochastic and adversarial settings, without prior knowledge of the
regime and its parameters. Recent advances in this line of work
\citep{amir2020prediction,zimmert2019optimal,jin2020simultaneously,zimmert2019beating,lee2021achieving} have resulted with algorithms that achieve a remarkable \emph{best-of-all-worlds} guarantee: in a general intermediate adversarially-corrupted regime, where losses are
generated stochastically but then may be modified by an adversary in an
arbitrary way, these algorithms attain regret that grows with the square-root of
the total amount of corruption introduced. Even more, these guarantees
significantly outperform those obtained by specialized algorithms designed
exclusively for the corrupted stochastic setting
\citep{lykouris2018stochastic,gupta2019better,lykouris2019corruption}.

Our goal in this work is to extend these best-of-all-worlds results from the
full-information and MAB settings to more general online learning problems and
feedback models. We focus on the general framework of online learning with
graph-structure feedback, originally introduced by \citet{mannor2011bandits} and
studied extensively since (e.g.,
\cite{alon2015online,cohen2016online,kocak2016online,caron2012leveraging,lu2021corrupted,lu2021stochastic,hu2020problem}). In this model,
feedback is described by an undirected graph $G$ over the arms; the set of
observations received following each prediction is comprised of the loss of the
chosen arm as well as the loss of all of its neighbors in $G$.
The feedback graphs framework nicely interpolates between online learning with
full-information and learning with bandit feedback, and generalizes to settings
with more complex side-observation structure. Full-information online learning
is captured as an extreme case where the graph $G$ is the clique over the $N$
arms; on the other extreme, an empty feedback graph that contains no edges
corresponds to the multi-armed bandit problem.%

The optimal regret in the feedback graphs setting is well understood, both in
the adversarial and stochastic regimes. In the former, the state-of-the-art is
obtained by \citet{alon2015online} who gave an algorithm with regret
$\smash{\mathcal{O}(\sqrt{\alpha(G) T})}$, where $\alpha(G)$ is the independence number of
$G$, and also established its optimality up to log factors. In the latter
regime, \citet{cohen2016online} proposed an algorithm whose regret is roughly
$\mathcal{O}(\alpha(G)\log{T})$, which is again nearly optimal. However, it remains
unclear whether there exists a single algorithm that attains both bounds
simultaneously, without prior knowledge of the regime, considering that the
existing optimal algorithms in the two settings are inherently different. Of
course, one can always ignore the additional feedback and use an existing
best-of-all-worlds algorithm for MAB; however, this would result with a direct
dependence on the number of arms $N$. The challenge is then to improve this
dependence by exploiting the structure of the feedback graph.

\subsection{Our contributions}

We make a significant step towards obtaining a best-of-all-worlds guarantee in
the feedback graphs model, by giving the first algorithm in this setting that
achieves a uniform best-of-all-worlds guarantee and improves over the existing
MAB bounds in a non-trivial way. 
Our algorithm achieves regret of the form $\covnum(G)\mathrm{polylog}(T)$ when
the losses are stochastic i.i.d., and $\smash{\Otilde(\sqrt{\covnum(G) T})}$ when the
losses are fully adversarial; here $\covnum(G)$ denotes the \emph{clique covering
number} of the graph $G$.\footnote{The clique covering number of a graph is the
minimum number of cliques that cover its set of vertices; without loss of
generality, we can assume that the latter set of cliques is a partition.} 
Furthermore, in an intermediate adversarially-corrupted stochastic setting,
where the losses are generated stochastically but then may be corrupted by an
adversary, we obtain a bound of the form $\covnum(G)\mathrm{polylog}(T) +
\smash{\Otilde(\sqrt{\covnum(G) C})}$ with $C$ being the total amount of corruption, that
smoothly interpolates between these two extremes.
Crucially, the algorithm achieves the aforementioned bounds simultaneously and
agnostically, in the sense that it need not be aware of the actual underlying
loss generation process, and in particular, of the actual level of adversarial
corruption introduced.

To state our results more concretely, let $V_1,\ldots,V_K$ be a minimum clique
cover of $G$ (here $K = \covnum(G)$), and denote by $\Delta_k$ the minimal gap of
a suboptimal arm in $V_k$; namely, the difference in mean losses (with respect
to the stochastic uncorrupted losses) between an arm in $V_k$ and the best arm
whose mean loss is minimal. Finally, define $Z = \deltasum{1}$ (this is the
quantity governing the complexity of the learning problem in the purely
stochastic case; see \cite{cohen2016online}). Then, our main result can be
stated as:

\begin{theorem*}[main, informal]
There exists an algorithm (see \cref{alg:FTRL} in \cref{sec:algorithm}) whose expected regret in the $C$-corrupted stochastic setting is bounded by
$
    \smash{\Otilde\brk{\min\set{\sqrt{KT}, Z + \sqrt{CZ}}}}
    .
$
In particular, in the stochastic setting ($C=0$) the expected regret of the algorithm is at most 
$
    \smash{\Otilde\brk{\deltasum{1}}}
    ,
$ 
while in the adversarial setting ($C=T$) its expected regret is bounded by 
$
    \smash{\Otilde\brk{\sqrt{KT}}}
    .
$
\end{theorem*}

The above results leverage the graph structure in a non-trivial way and improve
the direct dependence on the number of arms $N$ in existing best-of-all-worlds
results to a dependence on $\covnum(G)$, which is potentially significantly
smaller than $N$. On the other hand, our bounds do not feature the optimal
dependence on the independence number $\alpha(G)$; we leave it as a main open
question whether it is possible to obtain a best-of-all-worlds result with
$\alpha(G)$ replacing $\covnum(G)$ in the bounds. We remark however, as we
discuss in more detail below, that obtaining bounds scaling with the
weaker~$\covnum(G)$ is already a highly non-trivial task and so we believe that
our results form a significant step towards obtaining optimal best-of-all-worlds
bounds. Another question we leave open is whether our dependence on~$T$ can be
improved to $\mathcal{O}(\log{T})$ and $\smash{\mathcal{O}(\sqrt{T})}$ in the stochastic and
adversarial cases respectively, removing a few excess logarithmic factors in our
bounds.

\subsection{Overview of main ideas and techniques}

Our approach is inspired by recent progress in obtaining best-of-all-worlds
guarantees in the two extreme cases of our problem: multi-armed bandits
\cite{zimmert2019optimal} and full-information experts
\cite{amir2020prediction}. Somewhat surprisingly, it has been shown in both
cases that this type of guarantee can be obtained by simple algorithms based on
the canonical Follow-the-Regularized-Leader template, traditionally used
exclusively in adversarial online learning. Our primary technical challenge
therefore boils down to designing a convex regularizer given the feedback graph
$G$, that carefully interpolates between the well-understood regularization
schemes in the full-information and bandit cases.

In the full-information case, best-of-all-worlds results
\cite{amir2020prediction} rely on (a time-varying version of) the negative
Shannon entropy 
$
    R(p) 
    = 
    \smash{ \sum_{i=1}^N p_i \log{p_i} }
$ 
as regularization. In a nutshell, its crucial properties are an $\mathcal{O}(\log{N})$ uniform upper bound on its
magnitude over the probability simplex, and on the other hand, that its Hessian
$\nabla^2 R(p) = \smash{\diag(p_1^{-1},\ldots,p_N^{-1})}$ satisfies $\ell\tr
(\nabla^2 R(p))^{-1} \ell = \mathcal{O}(1)$ for any loss vector $\ell \in [0,1]^N$. (We omit more technical details on third-order conditions and self-bounding
properties from this informal discussion.)
In the bandit case, on the other hand, the best-of-all-worlds result of
\cite{zimmert2019optimal} relies on a Tsallis entropy (with parameter
$\alpha=1/2$) regularizer 
$
    R(p) 
    = 
    \smash{-\sum_{i=1}^N \sqrt{p_i}}
    ,
$ 
originally proposed in the context of MAB by \cite{audibert2009minimax}. This form of entropy is
bounded uniformly over the simplex by $\smash{\mathcal{O}(\sqrt{N})}$, and its Hessian
$
    \smash{\nabla^2 R(p)}
    = 
    \smash{\tfrac14\diag(p_{\smash{1}}^{-3/2},\ldots,p_{\smash{N}}^{-3/2})}
$ 
satisfies
$
    \E\brk[s]{\tilde\ell\tr (\nabla^2 R(p))^{-1} \tilde\ell} 
    = 
    \smash{\mathcal{O}(\sqrt{N})}
$ 
for loss estimators $\tilde\ell$ used in MAB. In both cases, the regularizer
strikes the ``right'' balance between the two bounds, which is necessary for
optimal regret.

Moving to the more complex feedback model induced by a graph, it is natural to
expect that an appropriate regularization would emerge as a certain combination of the
Shannon and Tsallis entropies. For simplicity, consider a simple graph formed as
a union of $K$ disjoint cliques $V_1,\ldots,V_K$, and let $p$ be a probability
vector over its vertices. Then, we are looking for a regularization that would
treat the marginal clique probabilities $q_k = \sum_{i \in V_k} p_i$ as a Tsallis entropy,
while behaving within a clique (with respect to the internal conditional
probabilities $p_i/q_k$ for $i \in V_k$) like the negative Shannon entropy.
Inspecting the technical conditions more deeply, it turns out that we seek a
regularizer $R$ such that the Hessian at $p$ is lower bounded by a diagonal
matrix with the $i$'th diagonal entry corresponding to $i \in V_k$ is
$
    \smash{p_{\smash{i}}^{-1} q_{\smash{k}}^{-1/2}}
    .
$
A natural candidate would be simply the sum of the Shannon and Tsallis entropies; unfortunately, this regularization lacks the bi-level structure we need, and indeed, its Hessian does not admit the desired lower bound.

Our main technical innovation is the design of a new convex regularizer that admits the aforementioned lower bound on its Hessian, and on the other hand, is bounded by $\mathcal{O}(\smash{\sqrt{K}\log{N}})$ over the simplex. Interestingly, we find that the necessary conditions are met for a regularizer produced by the \emph{product} of a negative Tsallis entropy (over the marginals $q_k$) and a negative Shannon entropy (over the conditional $p_i/q_k$), rather than their sum.
Of course, such a product need not be a convex function in general, and indeed, the regularization we just described fails to be convex even in simple cases (see \cref{fig:3d}). Nevertheless, it turns out that there is a simple fix that makes this bi-level regularization a convex function: a simple linear shift to the Shannon entropy component. Not only that, but the resulting regularizer also admits the stronger Hessian lower bound we require. 
We refer to this new regularizer by \emph{Tsallis-Shannon entropy}, and discuss the details of its derivation in \cref{sec:tsallis-shannon}; the remainder of the  development takes more standard lines and is described in \cref{sec:algorithm,sec:corrupted-bound-proof}.

\subsection{Additional related work}

The online learning with feedback graphs framework we consider here was introduced by \citet{mannor2011bandits}. A tight characterization of the minimax regret in this model in the adversarial case was established by \cite{alon2017nonstochastic,alon2015online}.
The setting was first considered in the stochastic setting by \citet{caron2012leveraging}, and was followed by numerous subsequent papers \citep{cohen2016online,hu2020problem,lu2021stochastic}. \citet{cohen2016online} considered a setting in which the feedback can change in each time step, and the sequence of feedback graphs is not known to the algorithm, and \citet{lu2021stochastic} considered a setting where the loss distribution can change across rounds.
We remark that some of the mentioned work (most notably~\cite{alon2015online,cohen2016online}) analyze directed feedback graphs which are more general than the undirected graphs we consider here, and moreover, do not necessarily contain self-loops. Such models currently lie beyond our scope, but we believe that our regularization techniques should be extensible to these variants too. 

The line of work on best-of-both-worlds algorithms in the context of MAB was initiated by \citet{bubeck2012best}. Their result was subsequently improved in a sequence of papers \citep{auer2016algorithm,seldin2014one,wei2018more}, culminating in the remarkably elegant recent result of \cite{zimmert2019optimal}.
A number of tricks and techniques we use are borrowed (and sometimes extended)
from this line of work, including refined regret bounds for FTRL with local
norms \citep{amir2020prediction,zimmert2019optimal,jin2020simultaneously}, shifted loss estimators and self-bounding regret
\citep{wei2018more,zimmert2019optimal,amir2020prediction}, and augmented
log-barrier regularization for inducing stability
\citep{bubeck2018sparsity,lee2020closer,jin2020simultaneously}.
The adversarially corrupted stochastic setting had also received considerable attention in recent years across numerous online learning frameworks \citep{lykouris2018stochastic,lykouris2019corruption,gupta2019better,jun2018adversarial,kapoor2019corruption,liu2019data,zimmert2019optimal,jin2020simultaneously}, including a paper by \citet{lu2021corrupted} who analyzed this setting in the bandits with graph feedback framework and showed a regret bound which generalizes naturally to the stochastic setting, but not to the adversarial setting. Thus the question of whether it is possible to obtain regret bounds which generalize to best-of-both-worlds type guarantees in the feedback graphs framework, was not resolved in these papers.

Finally, we note that the idea of using a hybrid regularization was suggested in a number of previous works. \citet{bubeck2018sparsity,bubeck2019improved} explored hybrid regularizers in the context of multi-armed bandits. More recent works include \citet{zimmert2019beating} who used a hybrid regularization closely related to Tsallis entropy presented in \citet{zimmert2019optimal} to obtain best-of-both-worlds type guarantees in the semi-bandit setting. Similarly, \citet{jin2020simultaneously} use a form of a hybrid regularizer also related to Tsallis entropy in order to obtain similar guarantees for episodic reinforcement learning.
However, to the best of our knowledge, our idea and analysis of a multi-level regularization that mixes together distinct types of entropies, were not previously explored.

\section{Preliminaries} 
\label{sec:prelims}

We consider an online learning problem where the learner has to repeatedly
choose an arm from a set of $N$ arms (or actions) indexed by $[N] =
\brk[c]{1,2,...,N}$. In each time step $t = 1,2,...,T$ the learner generates a
probability vector $p_t$ from the $N$-dimensional simplex $\simplex_N =
\brk[c]{p \in \R_+^N : \smash{\sum_{i=1}^N} p_i = 1}$ and then chooses an arm $I_t \sim
p_t$. Thereafter, a loss vector $\ell_t \in [0,1]^N$ is generated, the learner
incurs the loss $\ell_{t,I_t}$, and feedback is revealed.

\paragraph{Feedback model:}

The feedback received by the learner at each round $t$ is specified by an
undirected \emph{feedback graph} $G = \brk{[N], E}$, known to the learner in
advance, and is comprised of $\brk[c]{(i,\ell_{t,i}) : i \in \neighbors(I_t)}$,
where $\neighbors(j)$ denotes the set of neighbors of node $j$ in $G$ (which we
assume always to contain $j$ itself). Moreover, we assume that the learner
receives as input a minimum \textit{clique covering} of $G$, i.e., a minimum
cardinality collection of cliques $\set*{\clique_1,\clique_2,...,\clique_K}$
which partitions and spans the vertices of $G$, meaning each $i \in [N]$ belongs
to some unique clique $\clique_k \subseteq [N]$. The number of cliques $K$ in
this minimum clique covering is called the \textit{clique covering number} of
$G$, which we denote by $\covnum(G)$.%
\footnote{Computing a minimum clique covering in general graphs is a well-known NP-hard problem; we therefore assume it is given as an input, to avoid dealing with such computational efficiency issues.}

\paragraph{Assumption on losses:}

We consider a general \emph{adversarially-corrupted stochastic} setting (originally introduced in \cite{lykouris2018stochastic}) that includes stochastic and adversarial online learning as special cases, as well as a range of intermediate problems.
In this setting, loss vectors $\iidell_1,\iidell_2,...,\iidell_T$ are first
drawn i.i.d.~from a fixed probability distribution, unknown to the learner. We
denote the mean loss vector by $\mu = (\mu_1,\mu_2,...,\mu_N) = \E[\iidell_t]$,
and let $i^\star = \argmin_i \mu_i$ the best arm, which we assume to be unique.
We further let $\clique_{k^\star} = \clique(i^\star)$ and denote by
$\Deltaarm{i} = \mu_i - \mu_{i^\star}$ the gap between arm $i$ and the best arm,
and $\Deltaclique{k} = \min_{i \in \clique_k, i \neq i^\star} \Deltaarm{i}$ the
minimal gap of a suboptimal arm which belongs to~$\clique_k$ (here, if
$V(i^\star)$ is a singleton then we interpret $\Deltaclique{k^\star} = 0$).
After the stochastic loss vectors $\iidell_1,\ldots,\iidell_T$ have been
generated, an adversary is allowed to corrupt them and form a final loss
sequence $\corell_1,\ldots,\corell_T \in [0,1]^N$. The overall (expected)
\emph{corruption level} introduced by the adversary is defined as
$$
    C 
    = 
    \E\brk[s]*{ \sum_{t=1}^T \norm{\corell_t - \iidell_t}_\infty }
    .
$$
Importantly, we assume that this quantity is unknown to the learner, and the
algorithms we design will be oblivious to the value of $C$.

\paragraph{Regret:}

The goal of the learner in the adversarially-corrupted stochastic setting is to minimize the \textit{pseudo-regret}, defined as follows:
\begin{align*}
    \regret 
    \eqdef 
    \E\brk[s]*{ \sum_{t=1}^T \corell_{t,I_t} } - \min_{i \in [N]} \E\brk[s]*{ \sum_{t=1}^T \ell_{t,i} }
    ~=~
    \E\brk[s]*{ \sum_{t=1}^T p_t \cdot \corell_t } - \min_{i \in [N]} \E\brk[s]*{ \sum_{t=1}^T \ell_{t,i} }
    .
\end{align*}
Here, the expectation is over the internal randomness of the algorithm and over any randomness in the losses. 
%
The \emph{stochastic} setting is obtained for corruption level $C = 0$, where the pseudo-regret simplifies to 
$
    \regret
    =
    \sum_{t=1}^T \sum_{i=1}^N \E[p_{t,i}] \Deltaarm{i}
    .
$
The \emph{adversarial} setting is recovered by setting $C = T$; with an oblivious adversary (that sets the loss vectors ahead of time), the loss vectors are w.l.o.g.~deterministic and the above definition of pseudo-regret matches the usual definition of regret.

\paragraph{Additional notation:}

Given $i \in [N]$ and a partition $V_1,\ldots,V_K$ of $[N]$, we denote by $V(i)$
the (unique) partition element which $i$ belongs to, i.e., $i \in V(i)$. For a
vector $p \in \R^N$ and a set of arms $A \subseteq [N]$, we use the shorthand
$p(A) = \sum_{i \in A} p_i$.

\section{The Tsallis-Shannon Entropy}
\label{sec:tsallis-shannon}

In this section we introduce our definition of the Tsallis-Shannon entropy that
we use as a regularizer and derive the main properties we require. For that, we
first define the notion of the Tsallis-perspective of a convex function, that
will be useful for this derivation.

\subsection{Tsallis-perspective of a convex function}

Let $h : [0,1] \to \R$ be twice-differentiable and strictly convex.
We define the \emph{Tsallis-perspective} of $h$ as the following function over $\R^d_+$:
\begin{align} \label{eq:tsallis-perspective}
    H(x)
    \eqdef
    \sqrt{\norm{x}_1} \sum_{i=1}^d h\brk2{ \frac{x_i}{\norm{x}_1} }
    .
\end{align}
The name means to draw a connection to the classical notion of the perspective of a convex function~$h$, given by $g(x,t) = t h(x/t)$, which is always convex in the pair $(x,t)$; using this fact, the convexity of the function $G(x) = \norm{x}_1 \sum_{i=1}^d h\brk{ \ifrac{x_i}{\norm{x}_1} }$ immediately follows since $x \mapsto \norm{x}_1$ is linear over $\R^d_+$.
In contrast, for the similarly looking function $H$, that has the leading $\norm{x}_1$ factor under a square-root, an analogous result does not hold in general; see \cref{fig:3d} for an example.
\begin{figure}[t]
    \centering
    \includegraphics[width=0.4\linewidth]{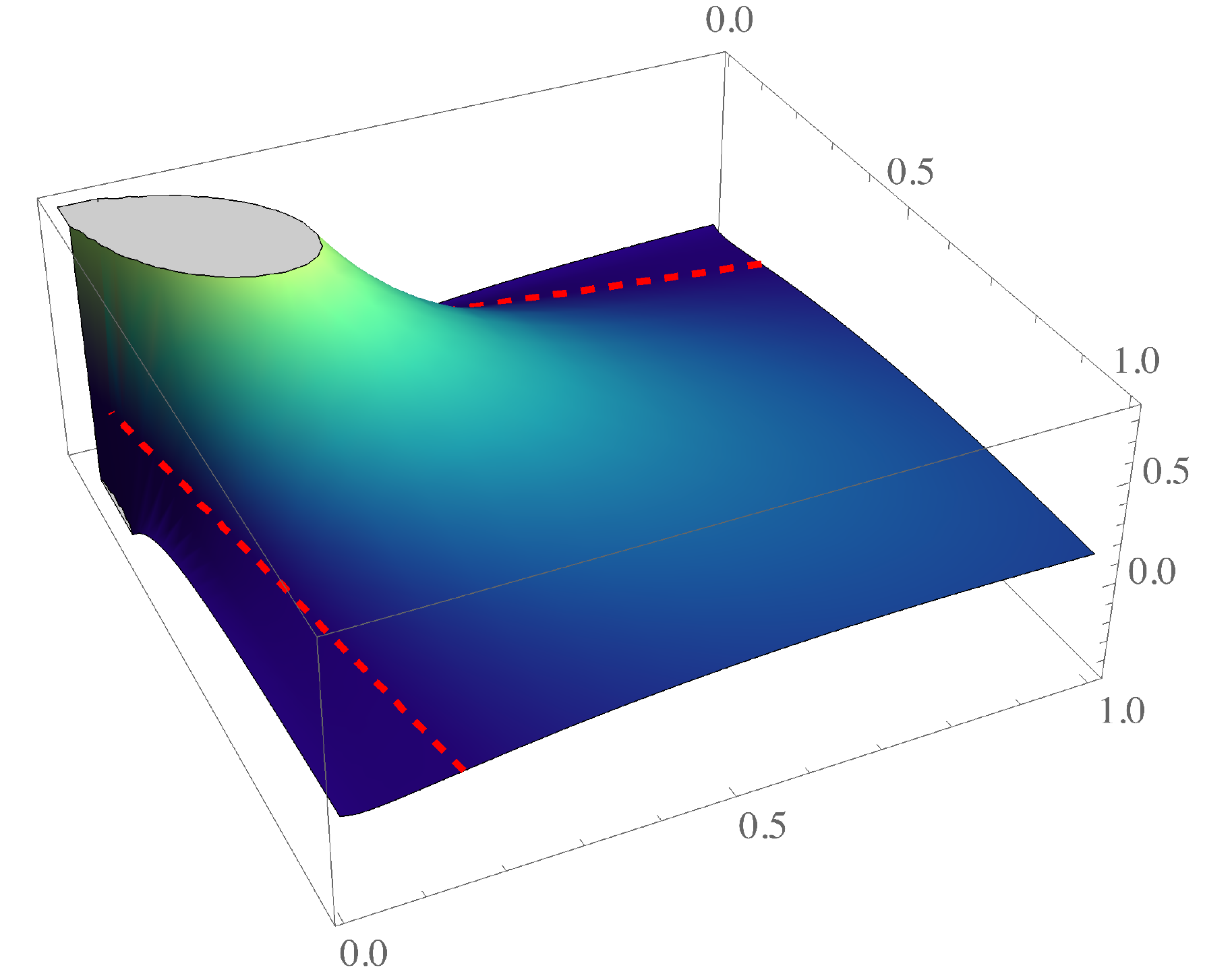}
    \qquad
    \includegraphics[width=0.4\linewidth]{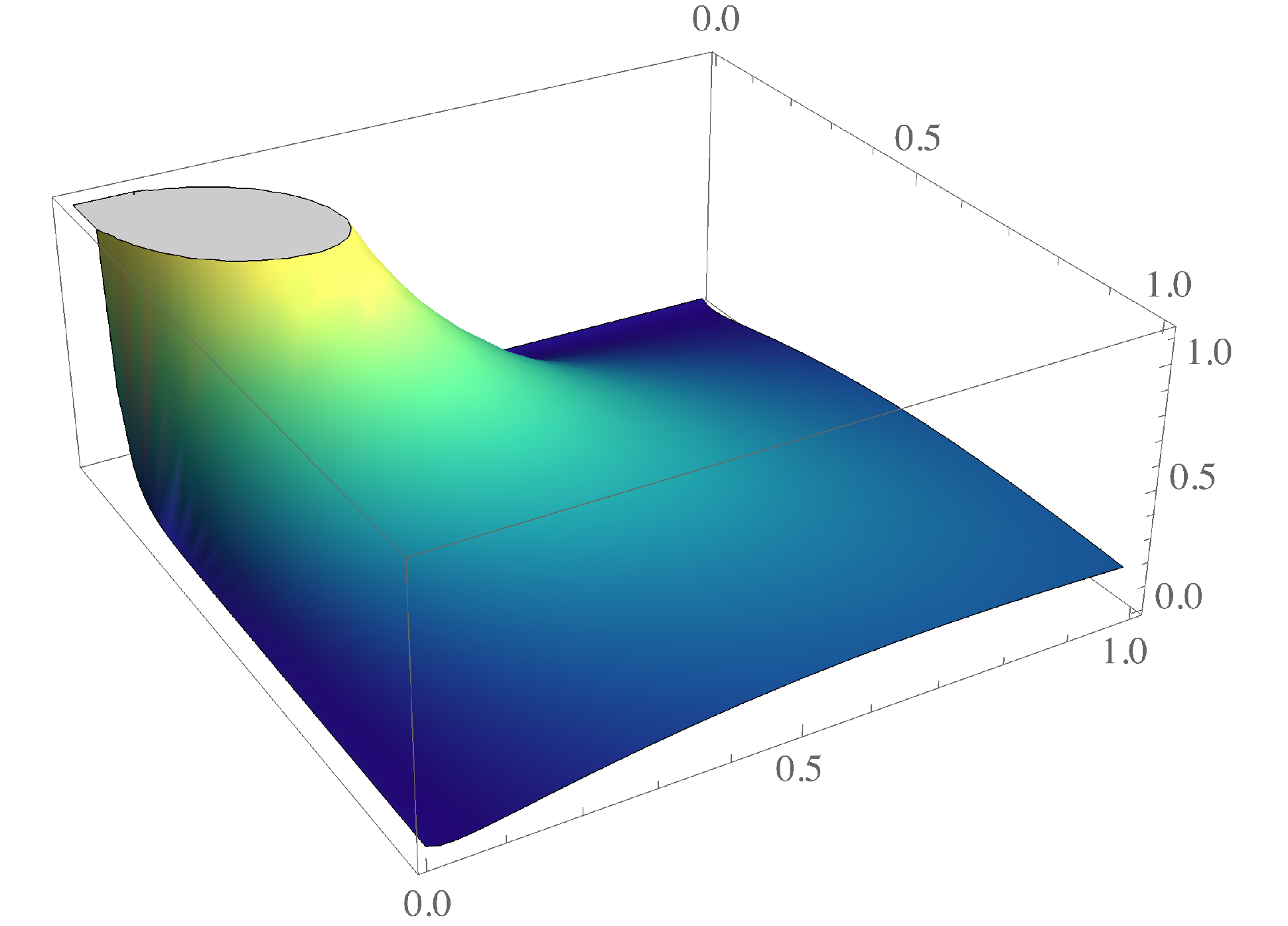}
    \caption{The minimal eigenvalue of the Hessian of the Tsallis-perspective in
    $d=2$ dimensions (see \cref{eq:tsallis-perspective}) of $h=\psi_\alpha$ (see
    \cref{eq:tsallis-shannon}) with $\alpha = 0$ (left) vs.~$\alpha = 0.25$
    (right). The dashed red lines indicate the zero level set. We see that
    convexity of $\psi_\alpha$ is insufficient to ensure that $\Psi_\alpha$ is
    even convex (let alone strongly convex), and the linear shift introduced in
    $\psi_\alpha$ is crucial for this purpose.}
    \label{fig:3d}
\end{figure}
Remarkably, however, there is a simple condition on $h$ under which it can be shown that not only $H$ is convex, but in fact admits a strong lower bound over its Hessian. This is detailed in the following lemma.

\begin{lemma} \label{lem:amazing2}
Assume that $h : [0,1] \to \R$ is twice-differentiable, strictly convex and satisfies the following condition for some constants $c_h \in \R$ and $\lambda_h > 0$, for all vectors $y \in Y \subseteq \simplex_d$:
\begin{align*}
    \sum_{i=1}^d h(y_i) + 2\sum_{i=1}^d \frac{\brk!{h'(y_i)-c_h}^2}{h''(y_i)} + \lambda_h
    \leq
    0
    .
\end{align*}
Then, at any $x \in \R^d_+$ such that $x/\norm{x}_1 \in Y$, the function $H$ satisfies
\begin{align*}
    \nabla^2 H(x)
    \succeq
    \frac{\lambda_h}{4} \norm{x}_1^{-\frac32} J + \frac{1}{2} \norm{x}_1^{-\frac72} \sum_{i=1}^d x_i^2 h'' \brk2{\frac{x_i}{\norm{x}_1}} z_i z_i \tr,
\end{align*}
where $J$ is the $d$-by-$d$ all-ones matrix and $z_i = \mathbf{1}_d - \brk{\norm{x}_1 / x_i} \mathbf{e}_i$ for $i \in [d]$.
\end{lemma}

\begin{proof}
Fix $x \in \R^d_+$ and let $y_i = x_i/\norm{x}_1$ for all $i$. Using \cref{lem:hessian-expression} in \cref{sec:amazing-proofs}, the Hessian of $H$ can be written as
\begin{align*}
    \nabla^2 H(x)
    =
    \frac14 \norm{x}_1^{-\tfrac32} \sum_{i=1}^d \brk2{- h(y_i) z z \tr + 4y_i^2 h''(y_i) z_i z_i \tr + 2 y_i h'(y_i) \brk!{z z_i \tr + z_i z \tr}},
\end{align*}
where $z = \mathbf{1}_d$ is the all-ones vector, and $z_i = \mathbf{1}_d - \brk{\norm{x}_1 / x_i} \mathbf{e}_i$ for all $i \in [d]$. Then, using the condition on $h$ and since $\sum_{i=1}^d y_i z_i = 0$ we have
\begin{align*}
    &\nabla^2 H(x)
    \\
    &\succeq
    \frac{\lambda_h}{4} \norm{x}_1^{-\tfrac32} J 
    + \frac14 \norm{x}_1^{-\tfrac32} \sum_{i=1}^d \brk*{ 
        \frac{\brk!{h'(y_i)-c_h}^2}{\frac12 h''(y_i)} z z\tr
        + 4 y_i^2 h''\brk{y_i} z_i z_i\tr
        + 2 y_i h'\brk{y_i} \brk!{ z z_i\tr + z_i z\tr }
    }
    \\
    &=
    \frac{\lambda_h}{4} \norm{x}_1^{-\tfrac32} J + \frac14 \norm{x}_1^{-\tfrac32} \sum_{i=1}^d 
    \brk*{ 
        \frac{\brk!{h'(y_i)-c_h}^2}{\frac12 h''(y_i)} z z\tr
        + 4 y_i^2 h''\brk{y_i} z_i z_i\tr
        + 2 y_i (h'\brk{y_i}-c_h) \brk!{ z z_i\tr + z_i z\tr }
    }
    \\
    &=
    \frac{\lambda_h}{4} \norm{x}_1^{-\tfrac32} J + \frac14 \norm{x}_1^{-\tfrac32} \sum_{i=1}^d 
        \brk4{ \frac{h'(y_i)-c_h}{\sqrt{\frac12 h''(y_i)}} z + 2 y_i \sqrt{\tfrac12 h''(y_i)} z_i }
        \brk4{ \frac{h'(y_i)-c_h}{\sqrt{\frac12 h''(y_i)}} z + 2 y_i \sqrt{\tfrac12 h''(y_i)} z_i }\tr
    \\
    &\quad+
    \frac{\lambda_h}{4} \norm{x}_1^{-\tfrac32} \sum_{i=1}^d y_i^2 h''(y_i) z_i z_i \tr
    ,
\end{align*}
and the result follows since each term in the first summation is psd.
\end{proof}

\subsection{Deriving the Tsallis-Shannon entropy}

We can now define the Tsallis-Shannon entropy in terms of Tsallis-perspectives and derive its local strong convexity properties that we require for our analysis.
For a given $\alpha>0$, the Tsallis-Shannon entropy with respect to a partition $V_1,\ldots,V_K$ of $V=[N]$ is defined for all $p \in \R^d_+$ as
\begin{align} \label{eq:tsallis-shannon}
    \Psi_\alpha(p) 
    = 
    \sum_{k=1}^K \sqrt{p(\clique_k)} \cdot \sum_{i \in \clique_k} \psi_\alpha\brk3{\frac{p_i}{p(\clique_k)}}
    ~,
    \quad\text{where}\quad
    \psi_\alpha(y)
    =
    y \log y - \alpha y
    ~.
\end{align}
Namely, $\Psi_\alpha$ is the sum of the Tsallis-perspectives of $\psi_\alpha$ with respect to each partition element~$V_k$. Observe that $\psi_\alpha$ corresponds to a single term of a (linearly shifted) negative Shannon entropy. 
Thus, for a probability vector $p$, the function $\Psi_\alpha$ can be viewed as a bi-level entropy, where the top level amounts to the marginal probabilities $p(V_k)$ (for $k \in [K]$) and the bottom level to the conditional probabilities $p_i/p(V_k)$ (for $i \in V_k$); for the marginal probabilities $\Psi_\alpha$ behaves like a $\tfrac12$-Tsallis entropy, while on the conditional probabilities it operates as a (negative, shifted) Shannon entropy.

It is not hard to show that the magnitude of $\Psi_\alpha(p)$ is at most
$\mathcal{O}(\sqrt{K}\log{N})$, which is crucial for our purposes as it avoids a direct
polynomial dependence on the dimension $N$. On the other hand, using
\cref{lem:amazing2}, we can prove the following lower bound for the Hessian of
$\Psi_\alpha$ for an appropriate choice of $\alpha$. We remark that the setting
of $\alpha$ is crucial: the linear shift in the Shannon entropy component is
essential even just for ensuring the convexity of $\Psi_\alpha$ (see
\cref{fig:3d}).

\begin{lemma} 
\label{lem:strong-convexity-reg} 
Fix $\gamma>0$, and let $p \in \R^N_+$ such that $p_i/p(V_k) \geq \gamma$ for
all $i,k$ with $i \in V_k$. Then for $\alpha = 2\brk{1 +
\log^2(\ifrac{1}{\gamma})}$, the Hessian of $\Psi_\alpha$ at $p$ satisfies
\begin{align*}
    \nabla^2 \Psi_\alpha(p)
    \succeq
    \frac12 \diag\brk*{
        \frac{1}{p_1 \sqrt{p(V(1))}},
        \frac{1}{p_2 \sqrt{p(V(2))}},
        \ldots,
        \frac{1}{p_N \sqrt{p(V(N))}}
    }
    .
\end{align*}
\end{lemma}

\begin{proof}
For brevity, we omit all $\alpha$ subscripts below.
Observe that $\Psi$ is a separable sum of $K$ functions, each of which is a function of variables $p_i$ with $i \in V_k$, thus its Hessian $\nabla^2 \Psi(p)$ is block-diagonal with blocks aligned with $V_1,\ldots,V_K$.
It therefore suffices to establish the Hessian bound for a function of the form $\Psi^\circ(x) = \sqrt{\norm{x}_1} \sum_{i=1}^d \psi\brk{ \ifrac{x_i}{\norm{x}_1} }$ corresponding to a set of variables $V = \set{x_1,\ldots,x_d}$.
To this end, we first show that 
the function $\psi$ satisfies the conditions of \cref{lem:amazing2} over the domain $Y = \simplex_d \cap \set{y : y_i \geq \gamma, ~\forall ~ i \in [d]}$, with constants $c_\psi = - \brk{1 + 2 \log^2(\ifrac{1}{\gamma})}$ and $\lambda_\psi = 2$.
Indeed, twice differentiability and strict convexity are immediate, and for any $y \in Y$
we have
\begin{align*}
    \sum_{i=1}^d \psi(y_i) + 2 \sum_{i=1}^d \frac{\brk!{\psi'(y_i) - c_\psi}^2}{\psi''(y_i)} + \lambda_\psi
    & \leq
    -2 \log^2 \tfrac{1}{\gamma} + 2 \sum_{i=1}^d y_i \brk!{\log y_i - 1 - 2 \log^2 \tfrac{1}{\gamma} - c_\psi}^2
    \\
    &=
    -2\log^2 \tfrac{1}{\gamma} + 2 \sum_{i=1}^d y_i \log^2 y_i
    \\
    &\leq
    0
    .
\end{align*}
Applying the latter lemma on $\Psi^\circ$, we obtain (below $J$ denotes the $d$-by-$d$ all-ones matrix):
\begin{align*}
    \nabla^2 \Psi^\circ(x)
    &\succeq
    \frac12 \norm{x}_1^{-\frac32} J + \frac{1}{2} \norm{x}_1^{-\frac72} \sum_{i=1}^d x_i^2 \psi'' \brk2{\frac{x_i}{\norm{x}_1}} z_i z_i \tr
    \\
    &=
    \frac12 \norm{x}_1^{-\frac32} J + \frac12 \norm{x}_1^{-\frac52} \sum_{i=1}^d x_i z_i z_i \tr
    \\
    &=
    \frac12 \norm{x}_1^{-\frac32} J + \frac12 \norm{x}_1^{-\frac32} J 
    - \frac12 \norm{x}_1^{-\frac32} \sum_{i=1}^d \brk!{\mathbf{1}_d \mathbf{e}_i \tr + \mathbf{e}_i \mathbf{1}_d \tr} 
    + \frac12 \norm{x}_1^{-\frac12} \sum_{i=1}^d \frac{1}{x_i} \mathbf{e}_i \mathbf{e}_i \tr
    \\
    &= 
    \norm{x}_1^{-\frac32} J - \frac12 \norm{x}_1^{-\frac32} \brk*{2 J} + \frac12 \norm{x}_1^{-\frac12} \diag{\brk*{\frac{1}{x_1},\frac{1}{x_2},...,\frac{1}{x_d}}}
    \\
    &= 
    \frac{1}{2\sqrt{\norm{x}_1}} \diag{\brk*{\frac{1}{x_1},\frac{1}{x_2},...,\frac{1}{x_d}}}
    ,
\end{align*}
and the result follows.
\end{proof}

\section{Algorithm and Main Result}
\label{sec:algorithm}

In this section we present our best-of-all-worlds algorithm for online learning with feedback graphs. The algorithm, detailed in \cref{alg:FTRL}, follows the general Follow-the-Regularized-Leader (FTRL) template, and is instantiated by a choice of time-varying convex regularization functions $R_t$ together with a standard loss estimator for graph-structured feedback due to \cite{alon2017nonstochastic}.

\begin{algorithm}[ht]
 \SetAlgoLined
 \textbf{Input:} clique covering $\brk[c]*{\clique_1,\clique_2,...,\clique_K}$ of an undirected feedback graph $G$\;
 let $\alpha = 2 \brk{\log^2(NT) + 1}$, $\beta = 9$, $\gamma = \ifrac{1}{(NT)}$, and step sizes $\eta_t = \ifrac{1}{\sqrt{t}}$\;
 initialize $\widehat{L}_{0,i} = 0$ for all $i\in[N]$\;
 \For{$t=1,2,...,T$} {
    update 
    $$
        p_{t} 
        = 
        \argmin_{p \in \simplexftrl} \set!{ \widehat{L}_{t-1} \dotp p + R_t(p)}
    $$
    where $R_t(p) = \eta_t^{-1} \Psi(p) + \Phi(p)$ ($\Psi$ and $\Phi$ are defined in \cref{eqn:reg,eqn:log-barrier})\;
    pick arm $I_t \sim p_t$, observe feedback $\set!{(i, \ell_{t,i}) : i \in \neighbors(I_t)}$ and set  
    $$
    \forall ~ i \in [N] : \qquad
    \estell_{t,i} = \frac{\ell_{t,i}}{p_t(V(i))} \mathbb{I}\{i \in V(I_t)\}
    \text{\;}
    $$
    update $\widehat{L}_t = \widehat{L}_{t-1} + \estell_t$\;
 }
 \caption{FTRL with feedback graphs}
 \label{alg:FTRL}
\end{algorithm}

Our main result regarding \cref{alg:FTRL} is given in the following.

\begin{theorem}
\label{thm:reg-corrupted}
\cref{alg:FTRL} attains the following expected pseudo-regret bound in the $C$-corrupted stochastic setting:
\begin{align*}
    \regret
    &=
    \Otilde\brk3{
        \min\brk[c]3{ \sqrt{KT}, \deltasum{1} + \sqrt{C \, \deltasum{1}} }
    }
    .
\end{align*}
\end{theorem}

The regularizer $R_t$ is formed as a sum of two convex functions, $\Psi$ and $\Phi$, described below. Predictions are computed by the algorithm by minimizing the cumulative estimated loss, regularized by $R_t$, over a truncated simplex 
$
    \simplexftrl
    \eqdef 
    \brk[c]{p \in \simplex_N : \forall i, ~ p_i \geq \gamma}
    .
$ 
For loss estimation, the algorithm relies on the standard unbiased loss estimators (denoted by $\estell_t$) for the graph-feedback framework. Note however that the while feedback the algorithm revealed at time step $t$ includes the losses of all the neighbors of $I_t$ in $G$, it only makes use of the losses of the arms which belong to the same clique as $I_t$; in other words, the algorithm may ignore some of the feedback it receives.

The primary regularization function $\Psi$ (whose weight increases with time) is the Tsallis-Shannon entropy $\Psi_\alpha$ defined in \cref{sec:tsallis-shannon}, given explicitly as
\begin{align}
\label{eqn:reg}
    \Psi(p) 
    = 
    -\alpha \sum_{k=1}^K \sqrt{p(\clique_k)} + \sum_{k=1}^K \frac{1}{\sqrt{p(\clique_k)}} \sum_{i \in \clique_k} p_i \log \frac{p_i}{p(\clique_k)}
    .
\end{align}
%
The second time-invariant regularizer $\Phi$ is a log-barrier of the marginal clique probabilities,
\begin{align}
\label{eqn:log-barrier}
    \Phi(p)
    =
    -\beta \sum_{k=1}^K \log p(\clique_k),
\end{align}
and its goal is to further stabilize the algorithm. While the idea of augmenting the regularizer with a log-barrier function for promoting stability is not new, we note that a standard log-barrier of the form 
$
    -\smash{\sum_{i=1}^N} \log p_i
$ 
would inevitably introduce an additive $\mathcal{O}(N)$ to the regret, which is suboptimal for our purposes. For that reason we use a different variant of a log-barrier function that operates on the marginal probabilities of the cliques.

\section[Proof of Main Theorem]{Proof of \cref{thm:reg-corrupted}}
\label{sec:corrupted-bound-proof}

In this section we provide details on the proof of our main theorem.
Our main step towards proving \cref{thm:reg-corrupted} is the following general regret bound for \cref{alg:FTRL}.

\begin{theorem} \label{thm:graph-regret-bound-main}
\cref{alg:FTRL} attains the following pseudo-regret bound, regardless of the corruption level:
\begin{align}
\label{eqn:graph-reg-bound-main}
    \regret
    = 
    \Otilde
    \brk*{
    K 
    + \sum_{t=1}^T \sumcliqueneq \sqrt{ \frac{\E\brk[s]*{\ptclique{t}{k}}}{t}} 
    + \sum_{t=1}^T \sqrt{\frac{\E[p_t \brk*{\clique_{k^\star} \!\setminus\! i^\star}]}{t}}
    + \sum_{t=1}^T \sqrt{\frac{\E[\phalfcliqueminus{k^\star}{i^\star}]}{t}}
    }
    ,
\end{align}
where $\phalf = \argmin_{p \in \simplexftrl} \set{\widehat{L}_t \dotp p + R_t(p)}$ for all $t$.
\end{theorem}


The focus of this section is on proving \cref{thm:graph-regret-bound-main}, but
first we sketch how it implies our main result.

\begin{proof}[Proof of \cref{thm:reg-corrupted} (sketch; full proof in \cref{sec:corrupted-bound-proof-full})]
The worst-case bound of $\Otilde \brk{\sqrt{KT}}$ follows immediately from the fact that $\sum_{t=1}^T \sum_{k=1}^K \sqrt{\E[\ptclique{t}{k}] / t} \leq 2 \sqrt{KT}$ via Jensen's inequality, and similarly for the term including $p_t^+$.
The bound involving the corruption level $C$ requires a more delicate argument due to \cite{zimmert2019optimal} that makes use of a self-bounding property of the pseudo-regret. 
In more detail, using Young's inequality one has for all $z > 0$:
\begin{align*}
    \sum_{t=1}^T \sumcliqueneq \sqrt{\frac{\E[\ptclique{t}{k}]}{t}} 
    &\leq
    \sum_{t=1}^T \sumcliqueneq \brk*{\ifrac{z}{2t \Deltaclique{k}} + \ifrac{\E[\ptclique{t}{k}] \Deltaclique{k}}{2z}},
\end{align*}
and a similar bound can be shown for the other two summations in \cref{eqn:graph-reg-bound-main}.
Combining these two gives after some simplification the pseudo-regret bound
$
    \regret
    \leq
    z B + z^{-1} (\regret + 2C),
$
for $B = \mathcal{O}\brk{\deltasum{\log(T)}}$, which further simplifies to $\regret \leq 2B + (z-1)B + \frac{2C + B}{z-1}$. 
Optimizing the bound with respect to $z$ then gives
$
    \regret
    \leq
    4B + 4 \sqrt{BC} 
    ,
$
which implies the bound we claimed.
\end{proof}

We now set to prove \cref{thm:graph-regret-bound-main}.
Our starting point is a general regret bound for FTRL. 
Applying a variation on the standard FTRL analysis (a similar analysis was used in \cite{zimmert2019optimal,jin2020simultaneously}) to the FTRL instance of \cref{alg:FTRL} gives the following regret bound.
\begin{lemma}
\label{lem:general-regret-bound}
For all $p^\gamma \in \simplexftrl$ the following holds:
\begin{align}
    \sum_{t=1}^T 
    \estell_t \cdot \brk*{p_t - p^\gamma} 
    &\leq
    \Phi(p^\gamma) - \Phi(p_1) + \sum_{t=1}^T \brk*{\frac{1}{\eta_t} - \frac{1}{\eta_{t-1}}} \brk*{\Psi(p^\gamma) - \Psi(p_t)}
    \label{eqn:ftrl-penalty}
    \\
    &+
    2 \sum_{t=1}^T \eta_t \brk*{\norm{\estell_t - 
    \ell_{t,i^\star} \mathbf{1}}_t^*}^2
    \label{eqn:ftrl-stability}
    .
\end{align}
Here $\norm{g}_t^* = \sqrt{g \tr (\grad ^2 \Psi(\tilde{p}_t))^{-1} g}$ is the dual local norm induced by $\Psi$ at $\tilde{p}_t$ for some intermediate point $\tilde{p}_t \in [p_t, \phalf]$, where $\phalf = \argmin_{p \in \simplexftrl} \set{\widehat{L}_t \dotp p + R_t(p)}$.
\end{lemma}

For the proof, see \cref{sec:general-regret-bound}.
The bound above can be seen as a sum of a \emph{penalty} term and a \emph{stability} term (RHS of \cref{eqn:ftrl-penalty,eqn:ftrl-stability} respectively). In order to prove \cref{thm:graph-regret-bound-main} using this general regret bound, we separately bound the penalty and stability terms. The penalty term is simpler to handle, and we only provide a brief discussion about how it is bounded. Bounding the stability term requires more work and constitutes the bulk of our analysis and our main contributions. We remark that the stability term contains dual norms of the loss estimators after introducing an additive shift of the form $\ell_{t,i^\star}$ to each arm. The purpose of such a shift is to exclude the contribution of the best arm $i^\star$ from the regret bound, as can be seen in \cref{eqn:graph-reg-bound-main}. More details regarding this shift and its purpose are given in \cref{sec:stability-proofs}. 

\paragraph{Bounding the penalty term:}

We present a brief discussion about bounding the \emph{penalty} term in \cref{eqn:ftrl-penalty}. The penalty term is handled using bounds on the magnitude of the regularization over the domain $\simplexftrl$. 
\begin{lemma}
\label{lem:penalty-main}
The penalty term (RHS of \cref{eqn:ftrl-penalty}) is upper bounded by 
\begin{align*}
    \Otilde \brk*{
    K
    +
    \sum_{t=1}^T \sum_{k \neq k^\star} \sqrt{\frac{\ptclique{t}{k}}{t}} 
    + 
    \sum_{t=1}^T \sqrt{\frac{p_t \brk*{\clique_{k^\star} \setminus i^\star}}{t}}
    },
\end{align*}
for $p^\gamma$ defined by
$
    p^\gamma_i = 
    \begin{cases}
    \gamma & i \neq i^\star \\
    1 - (N-1) \gamma & i = i^\star
    \end{cases}. \label{eqn:p-gamma}
$
\end{lemma}

\begin{proof}[Proof (sketch; full proof in \cref{sec:penalty-proofs})]
Essentially, we use the fact that the log-barrier component of the regularization is bounded by $\Otilde(K)$, and the Tsallis-Shannon component evaluated on the prediction $p_t$ is bounded by $\Otilde\brk!{\sum_{k=1}^K \sqrt{\ptclique{t}{k}}}$, as it can be seen as a weighted sum of entropy functions (bounded in magnitude by $\log N$), each of which corresponds to a clique $\clique_k$ with the respective weight being $\sqrt{\ptclique{t}{k}}$. The non-trivial part in bounding the penalty term is in omitting the term corresponding to $\clique_{k^\star}$ from the sum. 
This is accomplished by our specific choice of $p^\gamma$ described above, as well as carefully bounding the (Shannon) entropy terms within each clique.
\end{proof}

\paragraph{Bounding the stability term:}

The following lemma provides an upper bound on the stability term (RHS of \cref{eqn:ftrl-stability}). This lemma crucially relies on strong convexity properties of the Tsallis-Shannon regularization (\cref{eqn:reg}). These properties allow us to obtain highly non-trivial upper bounds on the eigenvalues of the inverse Hessian of $\Psi(\cdot)$ over $\simplexftrl$, and our ability to obtain such bounds turns out to be critical for us to be able to successfully bound the stability term.

\begin{lemma}
\label{lem:stability-term-bound}
The following holds for all time steps $t$:
\begin{align*}
    \E\brk[s]!{ \brk{\norm{\estell_t - \ell_{t,i^\star} \mathbf{1}}_t^*}^2 } 
    =
    \Otilde\brk4{
        \sum_{k \neq k^\star} \sqrt{\E\brk[s]!{\ptclique{t}{k}}}
        +
        \sqrt{\E\brk[s]!{\phalf \brk*{\clique_{k^\star} \!\setminus\! i^\star}}}
    }
    .
\end{align*}
\end{lemma}

\begin{proof}[Proof (sketch; full proof in \cref{sec:stability-proofs})]
For simplicity, in this proof sketch we ignore the shift in the loss estimation and only bound $\E\brk[s]{\brk{\norm{\estell_t}_t^*}^2}$ as it captures the main challenges and new ideas in the analysis. First, recall that the dual norm is formally defined with respect to the Hessian of $\Psi(\cdot)$ at some intermediate point $\tilde{p}_t \in [p_t, \phalf]$. 
We use the stability properties induced by the augmented log-barrier regularizer to replace $\tilde{p}_t(V_i)$ terms with $p_t(V_i)$ up to constant factors, and within each clique, bound the remaining terms involving $\tilde{p}_i$ by a sum of similar terms with the latter replaced by $p_i$ and $p^+_i$.

Next, for bounding the relevant dual norm, we require a suitable upper bound on 
$
    \smash{\brk{ \grad^2 \Psi(p_t) }^{-1}}
    .
$
To this end, we employ \cref{lem:strong-convexity-reg}, by which $\smash{\grad^2 \Psi(p_t)}$ is lower bounded by a diagonal matrix $D_t$, in which the $i$'th diagonal entry corresponding to $i \in \clique_k$ is of the form 
$
    \smash{ \brk!{2\sqrt{\ptclique{t}{k}} p_{t,i}}^{-1} }
    .
$ 
We now use this crucial fact to obtain the following upper bound on the stability term:
\begin{align*}
    \E\brk[s]!{\brk{\norm{\estell_t}_t^*}^2}
    \leq
    2\E \brk[s]*{\estell_t \tr D_t^{-1} \estell_t}
    =
    2\E \brk[s]*{\sum_{k=1}^K \sqrt{\ptclique{t}{k}} \sum_{i \in \clique_k} p_{t,i} \estell_{t,i}^2}.
\end{align*}
We keep bounding this term by using the definition of the loss estimators $\estell_t$ to obtain
\begin{align*}
    \E\brk[s]!{\brk{\norm{\estell_t}_t^*}^2}
    \leq
    2\E \brk[s]*{\sqrt{\ptcliquei{t}{I_t}} \sum_{\mathclap{i \in \clique(I_t)}} p_{t,i} \brk*{\frac{\ell_{t,i}}{\ptcliquei{t}{I_t}}}^2}
    \leq
    2\E \brk[s]*{\frac{1}{\sqrt{\ptcliquei{t}{I_t}}}}
    =
    2\E \brk[s]*{\sum_{k=1}^K \sqrt{\ptclique{t}{k}}},
\end{align*}
where in the last line we used the fact that the probability that at time step $t$ the algorithm chooses an arm from the clique $\clique_k$ is $\ptclique{t}{k}$. Note that this bound contains a term for $p_{t,i^\star}$ which we would like to omit from the final bound. As we remarked before, the way we mitigate that is by considering the shifted loss estimators. This is explained in full in \cref{sec:stability-proofs}. 
\end{proof}

\paragraph{Concluding the proof:}

We can now sketch the proof of \cref{thm:graph-regret-bound-main} using the bounds we obtained. A formal proof that also takes into account the poly-log factors can be found in \cref{sec:graph-regret-bound-proof}.
\begin{proof} [Proof of \cref{thm:graph-regret-bound-main} (sketch; full proof in \cref{sec:graph-regret-bound-proof})]
We first remark that it suffices to bound the expected regret with respect to $p^\gamma$ given by $\E \brk[s]!{\sum_{t=1}^T \ell_t \cdot \brk*{p_t - p^\gamma}}$, since it can only be larger than the pseudo-regret by an additive constant, as shown in \cref{sec:graph-regret-bound-proof}. In addition, since the loss estimators $\estell_t$ are unbiased estimators of the loss vectors $\ell_t$ we conclude that it suffices to bound the regret with respect to the loss estimators, i.e. $\E \brk[s]!{\sum_{t=1}^T \estell_t \cdot \brk*{p_t - p^\gamma}}$ which is bounded by the sum of the expected penalty and stability terms (\cref{eqn:ftrl-penalty,eqn:ftrl-stability}). 
\cref{lem:penalty-main} gives a bound on the penalty term of the form 
$$
    \Otilde \brk*{K + \sum_{t=1}^T \sum_{k \neq k^\star} \sqrt{\ifrac{\E[\ptclique{t}{k}]}{t}} + \sum_{t=1}^T \sqrt{\ifrac{\E[p_t \brk*{\clique_{k^\star} \!\setminus\! i^\star}]}{t}}}
    ,
$$ 
and \cref{lem:stability-term-bound} gives a bound on the stability term of the form
$$
    \Otilde \brk*{\sum_{t=1}^T \sum_{k \neq k^\star} \sqrt{\ifrac{\E[\ptclique{t}{k}]}{t}} + \sum_{t=1}^T \sqrt{\ifrac{\phalfcliqueminus{k^\star}{i^\star}}{t}}}
    .
$$ 
Adding the two bounds and rearranging, we conclude the proof of \cref{thm:graph-regret-bound-main}.
\end{proof}

\subsection*{Acknowledgements}

This work has received support from the Israeli Science Foundation (ISF) grant no.~2549/19, from the Len Blavatnik and the Blavatnik Family foundation, and from the Yandex Initiative in Machine Learning.

\bibliographystyle{abbrvnat}
\bibliography{bibliography}

\appendix


\section{Tsallis-perspective: Technical Proofs}
\label{sec:amazing-proofs}

To prove \cref{lem:amazing2} we need the following technical result that gives an expression for the Hessian of the Tsallis-perspective $H$ in terms of the (scalar) derivatives of $h$.

\begin{lemma} \label{lem:hessian-expression}
The Hessian of $H$ (\cref{eq:tsallis-perspective}) at any point $x \in \R_+^d$ can be expressed as:
\begin{align*}
    \nabla^2 H(x)
    =
    &-\frac14 \norm{x}_1^{-\tfrac32} \sum_{i=1}^d h\brk2{\frac{x_i}{\norm{x}_1}} z z\tr
    \\
    &+ \norm{x}_1^{-\tfrac72} \sum_{i=1}^d x_i^2 h''\brk2{ \frac{x_i}{\norm{x}_1} } z_i z_i\tr
    \\
    &+ \frac12 \norm{x}_1^{-\tfrac52} \sum_{i=1}^d x_i h'\brk2{ \frac{x_i}{\norm{x}_1} } \brk!{ z z_i\tr + z_i z\tr }
    ~,
\end{align*}
where $z = \mathbf{1}_d$ is the all-ones vector, and 
$
    z_i 
    = 
    \mathbf{1}_d - (\norm{x}_1 / x_i) \mathbf{e}_i
$
for all $i \in [d]$.
\end{lemma}

\begin{proof}
Let us first compute the first and second derivatives of $f(x) = \sqrt{\norm{x}_1}$ and $g_i(x) = h\brk{ \ifrac{x_i}{\norm{x}_1} }$ for a fixed $i \in [d]$:
\begin{align*}
    \nabla f(x)
    &=
    \frac12 \norm{x}_1^{-\tfrac12} z
    ;
    \\
    \nabla^2 f(x)
    &=
    -\frac14 \norm{x}_1^{-\tfrac32} zz\tr
    ;
    \\
    \nabla g_i(x)
    &=
    h'\brk2{ \frac{x_i}{\norm{x}_1} } 
    \brk2{ \frac{1}{\norm{x}_1} \mathbf{e}_i - \frac{x_i}{\norm{x}_1^{2}} z }
    \\
    &=
    -\frac{x_i}{\norm{x}_1^2} h'\brk2{ \frac{x_i}{\norm{x}_1} } z_i
    ;
    \\
    \nabla^2 g_i(x)
    &=
    h''\brk2{ \frac{x_i}{\norm{x}_1} } 
    \brk2{ \frac{1}{\norm{x}_1} \mathbf{e}_i - \frac{x_i}{\norm{x}_1^{2}} z } 
    \brk2{ \frac{1}{\norm{x}_1} \mathbf{e}_i - \frac{x_i}{\norm{x}_1^{2}} z }\tr
    \\
    &\quad+ 
    h'\brk2{ \frac{x_i}{\norm{x}_1} } 
    \brk2{ -\frac{1}{\norm{x}_1^{2}} z \mathbf{e}_i\tr - \frac{1}{\norm{x}_1^{2}} \mathbf{e}_i z\tr + \frac{2 x_i}{\norm{x}_1^{3}} zz\tr }
    \\
    &=
    \frac{x_i^2}{\norm{x}_1^4}
    h''\brk2{ \frac{x_i}{\norm{x}_1} } z_i z_i\tr
    +
    \frac{x_i}{\norm{x}_1^3}
    h'\brk2{ \frac{x_i}{\norm{x}_1} } \brk{ z z_i\tr + z_i z\tr }
    .
\end{align*}
Using the formula for the Hessian of a product, we now obtain:
\begin{align*}
    &\mkern-36mu
    \nabla^2 \brk!{ f(x) g_i(x) }
    \\
    &=
    \brk!{ \nabla^2 f(x) } g_i(x) + \nabla f(x) \nabla g_i(x)\tr + \nabla g_i(x) \nabla f(x)\tr + f(x) \brk!{ \nabla^2 g_i(x) }
    \\
    &=
    -\frac14 \norm{x}_1^{-\tfrac32} h\brk2{ \frac{x_i}{\norm{x}_1} } zz\tr
    -\frac12 \norm{x}_1^{-\tfrac52} x_i h'\brk2{ \frac{x_i}{\norm{x}_1} } \brk!{ z z_i\tr + z_i z\tr }
    \\
    &\phantom{=}+
    \norm{x}_1^{-\tfrac72} x_i^2 h''\brk2{ \frac{x_i}{\norm{x}_1} } z_i z_i\tr
    +
    \norm{x}_1^{-\tfrac52} x_i h'\brk2{ \frac{x_i}{\norm{x}_1} } \brk{ z z_i\tr + z_i z\tr }
    \\
    &=
    -\frac14 \norm{x}_1^{-\tfrac32} h\brk2{ \frac{x_i}{\norm{x}_1} } zz\tr
    +\norm{x}_1^{-\tfrac72} x_i^2 h''\brk2{ \frac{x_i}{\norm{x}_1} } z_i z_i\tr
    +\frac12 \norm{x}_1^{-\tfrac52} x_i h'\brk2{ \frac{x_i}{\norm{x}_1} } \brk!{ z z_i\tr + z_i z\tr }
    .
\end{align*}
Summing this over $i=1,\ldots,d$, we obtain the expression for the Hessian $\nabla^2 H(x)$.
\end{proof}

\section{Proof of Main Result}

In this section we provide the proof of \cref{thm:reg-corrupted}. 
In \cref{sec:q-stability} we prove useful lemmas which provide us with stability properties of the FTRL iterates. In \cref{sec:stability-proofs} and \cref{sec:penalty-proofs} we bound the stability and penalty terms (RHS of \cref{eqn:ftrl-stability} and \cref{eqn:ftrl-penalty}) towards proving \cref{thm:graph-regret-bound-main} in \cref{sec:graph-regret-bound-proof}. We then prove \cref{thm:reg-corrupted} in \cref{sec:corrupted-bound-proof-full}.

\subsection{Stability of Iterates}
\label{sec:q-stability}

We first establish a technical stability property of the FTRL updates that is crucial for bounding the stability term (\cref{eqn:ftrl-stability}). This property asserts that for every time step $t$, the clique marginal probabilities induced by $p_t$ are close, up to a constant multiplicative factor, to the clique marginals induced by $\phalf$, where $\phalf \eqdef \argmin_{p \in \simplexftrl} \set{\widehat{L}_t \cdot p + R_t(p)}$. The proof uses properties of the log-barrier component $\Phi$, and relies on an adaptation of an argument of \citet{jin2020simultaneously}.

\begin{lemma}
\label{lem:q-stability-new}
For all time steps $t$ and cliques $\clique_k$ it holds that
$
    \phalfclique{k} 
    \leq 
    \tfrac73 \ptclique{t}{k}
    ,
$
where 
$
    \phalf 
    \eqdef 
    \argmin_{p \in \simplexftrl} \brk[c]!{\widehat{L}_t \cdot p + R_t(p)}
    .
$
\end{lemma}

\begin{proof}
We define:
\begin{align*}
    F_t(p) &= \widehat{L}_{t-1} \cdot p + R_t(p),
    \\
    \Fhalf(p) &= \widehat{L}_t \cdot p + R_t(p),
\end{align*}
so that $p_t = \argmin_{p \in \simplexftrl} \brk[c]*{F_t(p)}$ and $\phalf = \argmin_{p \in \simplexftrl} \brk[c]*{\Fhalf(p)}$.
Note that $\grad^2 \Phi(p)$ is a block diagonal matrix, with the block corresponding to the clique $\clique_k$ being exactly $\frac{9}{\pclique{k}^2} J_{\clique_k}$ where $J_{\clique_k}$ is the $|\clique_k| \times |\clique_k|$ all-ones matrix. A straightforward calculation then shows that for all $p, p', p'' \in \simplex_N$ it holds that:
\begin{align*}
    \norm{p' - p''}_{\grad^2 \Phi(p)}^2 = 9 \sum_{k=1}^K \frac{(p'(\clique_k) - p''(\clique_k))^2}{\pclique{k}^2}.
\end{align*}
It suffices to prove that $\norm{\phalf - p_t}_{\grad^2 \Phi(p_t)}^2 \leq 16$. This is because by the calculation we just made, we have $\brk*{\phalfclique{k} - \ptclique{t}{k}}^2 \leq \brk{\frac43 \ptclique{t}{k}}^2$ which is want we want to prove. It then suffices to show that for any $p' \in \simplexftrl$ with $\norm{p' - p_t}_{\grad ^2 \Phi(p_t)}^2 = 16$ we have $\Fhalf(p') \geq \Fhalf(p_t)$. This is because as an implication of that, $\phalf$ which minimizes the convex function $\Fhalf$, must be within the convex set $\brk[c]!{p : \norm{p - p_t}_{\grad ^2 \Phi(p_t)}^2 \leq 16}$. We proceed to lower bound $\Fhalf(p')$ as follows:
\begin{align*}
    \Fhalf(p')
    &=
    \Fhalf(p_t) + \grad \Fhalf(p_t) \tr (p' - p_t) + \frac12 \norm{p' - p_t}_{\grad ^2 R_t(\xi)}^2 \\
    &=
    \Fhalf(p_t) + \grad F_t(p_t) \tr (p' - p_t) + \estell_t \tr (p' - p_t) + \frac12 \norm{p' - p_t}_{\grad ^2 R_t(\xi)}^2 \\
    &\geq
    \Fhalf(p_t) + \estell_t \tr (p' - p_t) + \frac12 \norm{p' - p_t}_{\grad ^2 \Phi(\xi)}^2,
\end{align*}
where the first equality is a Taylor expansion of $\Fhalf$ around $p_t$, with $\xi$ being a point between $p'$ and $p_t$, and the last inequality is due to first-order optimality conditions and the fact that $\grad ^2 R_t(\xi) \succeq \grad ^2 \Phi(\xi)$ since $\Psi$ is convex. Note that since $\norm{p' - p_t}_{\grad ^2 \Phi(p_t)}^2 = 16$, by the same argument as in the beginning of the proof we conclude that $p'(\clique_k) \leq \frac73 \ptclique{t}{k}$. Since $\xi$ lies between $p_t$ and $p'$ we conclude the same ratio bound for $\xi$. We can thus bound the last term as follows:
\begin{align*}
    \frac12 \norm{p' - p_t}_{\grad ^2 \Phi(\xi)}^2
    &=
    \frac92 \sum_{k=1}^K \frac{(p'(\clique_k) - \ptclique{t}{k})^2}{\brk{\xi(\clique_k)}^2}
    \\
    &\geq
    \frac{9}{2 \cdot \brk*{\frac73}^2} \sum_{k=1}^K \frac{(p'(\clique_k) - \ptclique{t}{k})^2}{\ptclique{t}{k}^2}
    \\
    &=
    \frac{9}{2 \cdot 49} \norm{p' - p_t}_{\grad ^2 \Phi(p_t)}^2
    \\
    &=
    \frac{72}{49}
    \geq
    1.
\end{align*}
It now suffices to show that $\estell_t \tr (p' - p_t) \geq -1$; indeed,
\begin{align*}
    \estell_t \tr (p' - p_t)
    =
    \sum_{i \in V(I_t)} \frac{\ell_{t,i}}{\ptcliquei{t}{I_t}} (p'_{t,i} - p_{t,i}) 
    \geq
    - \frac{1}{\ptcliquei{t}{I_t}}\sum_{i \in V(I_t)} \ell_{t,i} p_{t,i} 
    \geq
    -1
    ,
\end{align*}
and the proof is complete.
\end{proof}

The following lemma showcases another stability property that relates $p_t$ to $\phalf$. A corollary of this lemma is that the pseudo-regret of the iterates $p_t$ can only be larger than the pseudo-regret of the iterates $\phalf$, and it is used in the proof of \cref{thm:reg-corrupted} in \cref{sec:corrupted-bound-proof-full}.

\begin{lemma}
\label{lem:ftrl-shifted-regret}
For all time steps $t$ it holds that
\begin{align*}
    \phalf \cdot \estell_t \leq p_t \cdot \estell_t,
\end{align*}
where
$
    \phalf 
    \eqdef 
    \argmin_{p \in \simplexftrl} \brk[c]!{\widehat{L}_t \cdot p + R_t(p)}
    .
$
\end{lemma}

\begin{proof}
Since $\phalf$ is a minimizer of $\widehat{L}_t \cdot p + R_t(p)$ and $p_t$ is a minimizer of $\widehat{L}_{t-1} \cdot p + R_t(p)$, we have:
\begin{align*}
    \widehat{L}_t \cdot \phalf + R_t(\phalf) 
    &\leq 
    \widehat{L}_t \cdot p_t + R_t(p_t)
    \\
    &=
    \widehat{\ell}_t \cdot p_t + \widehat{L}_{t-1} \cdot p_t + R_t(p_t)
    \\
    &\leq
    \widehat{\ell}_t \cdot p_t + \widehat{L}_{t-1} \cdot \phalf + R_t(\phalf)
    ,
\end{align*}
and the claim follows by rearranging terms.
\end{proof}

\subsection[Proof of Stability Bound]{Proof of \cref{lem:stability-term-bound} (Stability)}
\label{sec:stability-proofs}

We now restate \cref{lem:stability-term-bound} which bounds the stability term to include extra constants which appear in the bound.

\begin{manualrestate}{\cref{lem:stability-term-bound}}[restated]
\label{lem:stability-term-bound-full}
The following holds for all time steps $t$:
\begin{align*}
    \E\brk[s]*{ \brk{\norm{\estell_t - \ell_{t,i^\star} \mathbf{1}}_t^*}^2 } 
    =
    56 \sum_{k \neq k^\star} \sqrt{\E\brk[s]!{\ptclique{t}{k}}}
    +
    8\sqrt{\E\brk[s]!{\phalf \brk*{\clique_{k^\star} \setminus i^\star}}}
    .
\end{align*}
Here $\norm{g}_t^* = \sqrt{g \tr (\grad ^2 \Psi(\tilde{p}_t))^{-1} g}$ is the dual local norm induced by $\Psi$ at $\tilde{p}_t$ for some intermediate point $\tilde{p}_t \in [p_t, \phalf]$, where $\phalf = \argmin_{p \in \simplexftrl} \set{\widehat{L}_t \dotp p + R_t(p)}$.
\end{manualrestate}

\begin{proof}
By \cref{lem:strong-convexity-reg}, $\grad ^2 \Psi(\tilde{p}_t)$ is lower bounded by a diagonal matrix $D_t$ in which the $i$'th diagonal entry corresponding to $i \in \clique_k$ is $\brk*{2\sqrt{\tilde{p}_t(\clique_k)} \tilde{p}_{t,i}}^{-1}$. Equivalently it holds that $\brk*{\grad ^2 \Psi(\tilde{p}_t)}^{-1} \preceq D_t^{-1}$. Using this fact and the fact that $\estell_{t,i} = 0$ for $i \notin \clique(I_t)$ we have
\begin{align}
    \E \left[ \brk*{\norm{\estell_t - \ell_{t,i^\star} \mathbf{1}}_t^*}^2 \right]
    &=
    \E \brk[s]*{\brk*{\estell_t - \ell_{t,i^\star} \mathbf{1}} \tr \brk*{\grad^2 \Psi(\tilde{p}_t)}^{-1} \brk*{\estell_t - \ell_{t,i^\star} \mathbf{1}}} \nonumber\\
    &\leq
    2\E\left[ \sum_{k=1}^K \sqrt{\tilde{p}_t(\clique_{k})} \sum_{i \in \clique_k} \tilde{p}_{t,i} \brk*{\estell_{t,i} - \ell_{t,i^\star}}^2 \right] \nonumber\\
    &=
    2\E \brk[s]*{\sqrt{\tilde{p}_t(\clique(I_t))} \sum_{i \in V(I_t)} \tilde{p}_{t,i} \brk*{\estell_{t,i} - \ell_{t,i^\star}}^2} \label{eqn:stability-chosen-clique}\\
    &+ 
    2\E \brk[s]*{\sum_{\clique_k \neq \clique(I_t)} \sqrt{\tilde{p}_t(\clique_k)} \sum_{i \in \clique_k} \tilde{p}_{t,i} (\ell_{t,i^\star})^2}
    , \label{eqn:stability-other-cliques}
\end{align}
where in the final equality we split the sum over cliques into a term for $\clique(I_t)$ and a sum over the rest of the cliques. We first show that the RHS of \cref{eqn:stability-chosen-clique} is bounded as follows:
\begin{align*}
    \E \brk[s]*{\sqrt{\tilde{p}_t(\clique(I_t))} \sum_{i \in V(I_t)} \tilde{p}_{t,i} \brk*{\estell_{t,i} - \ell_{t,i^\star}}^2}
    &\leq
    16 \sum_{k \neq k^\star} \sqrt{\E[\pclique{k}]} + 4\sqrt{\E \brk[s]*{\phalf \brk*{\clique_{k^\star} \setminus i^\star}}}.
\end{align*}
Indeed, due to \cref{lem:q-stability-new} and the fact that $\tilde{p}_t$ lies between $p_t$ and $\phalf$ it holds that $\tilde{p}_t(\clique_k) \leq 3 \ptclique{t}{k}$ for all $k$. Plugging in the expression for the loss estimator $\estell_t$ we obtain
\begin{align*}
    \E \left [ \sqrt{\tilde{p}_t \brk*{\clique(I_t)}} \sum_{i \in V(I_t)} \tilde{p}_{t,i} \brk*{\estell_{t,i} - \ell_{t,i^\star}}^2 \right]
    &=
    \E \left[ \sqrt{\tilde{p}_t \brk*{\clique(I_t)}} \sum_{i \in V(I_t)} \tilde{p}_{t,i} \brk*{\frac{\ell_{t,i}}{\ptcliquei{t}{I_t}} - \ell_{t,i^\star}}^2 \right] \\
    &\leq
    2 \E \left[ \ptcliquei{t}{I_t}^{-\tfrac32} \sum_{i \in V(I_t)} \tilde{p}_{t,i} \brk*{\ell_{t,i} - \ptcliquei{t}{I_t} \ell_{t,i^\star}}^2 \right] \\
    &=
    2\E \left[ \sum_{k=1}^K \ptclique{t}{k}^{-\tfrac12} \sum_{i \in \clique_k} \tilde{p}_{t,i} \brk*{\ell_{t,i} - \ptclique{t}{k} \ell_{t,i^\star}}^2 \right],
\end{align*}
where in the last equality we use the law of total expectation and the fact that conditioned on the history up until time step $t$ (including the decision vector $p_t$), the probability that $I_t$ belongs to the clique $\clique_k$ is exactly $\ptclique{t}{k}$.
In more detail:
\begin{align*}
    &\E \left[ \ptcliquei{t}{I_t}^{-\tfrac32} \sum_{i \in V(I_t)} \tilde{p}_{t,i} \brk*{\ell_{t,i} - \ptcliquei{t}{I_t} \ell_{t,i^\star}}^2 \right]
    \\
    &=
    \E \left[ \E_t \left[ \ptcliquei{t}{I_t}^{-\tfrac32} \sum_{i \in V(I_t)} \tilde{p}_{t,i} \brk*{\ell_{t,i} - \ptcliquei{t}{I_t} \ell_{t,i^\star}}^2 \right] \right] \\
    &=
    \E \brk[s]*{\sum_{k=1}^K \Pr[I_t \in \clique_k \mid h_t] \cdot \E_t \brk[s]*{\ptclique{t}{k}^{-\tfrac32} \sum_{i \in \clique_k} \tilde{p}_{t,i} \brk*{\ell_{t,i} - \ptclique{t}{k} \ell_{t,i^\star}}^2}} \\
    &=
    \E \brk[s]*{\sum_{k=1}^K \ptclique{t}{k} \cdot \E_t \brk[s]*{\pclique{k}^{-\tfrac32} \sum_{i \in \clique_k} \tilde{p}_{t,i} \brk*{\ell_{t,i} - \ptclique{t}{k} \ell_{t,i^\star}}^2}} \\
    &=
    \E \brk[s]*{\E_t \brk[s]*{\sum_{k=1}^K \ptclique{t}{k}^{-\tfrac12} \sum_{i \in \clique_k} \tilde{p}_{t,i} \brk*{\ell_{t,i} - \ptclique{t}{k} \ell_{t,i^\star}}^2}} \\
    &=
    \E \left[ \sum_{k=1}^K \ptclique{t}{k}^{-\tfrac12} \sum_{i \in \clique_k} \tilde{p}_{t,i} \brk*{\ell_{t,i} - \ptclique{t}{k} \ell_{t,i^\star}}^2 \right],
\end{align*}
where $h_t$ denotes the history up to and including the choice of $p_t$ at time step $t$ (not including the choice of $I_t$), and in the fourth equality we use linearity of expectation and the fact that $\ptclique{t}{k}$ is constant when conditioned on $h_t$. We proceed to bound the above term, while splitting the sum over cliques into a term for $\clique_{k^\star}$ and a sum for all of the other cliques:
\begin{align*}
    &\E \left[ \sum_{k=1}^K \ptclique{t}{k}^{-\tfrac12} \sum_{i \in \clique_k} \tilde{p}_{t,i} \brk*{\ell_{t,i} - \ptclique{t}{k} \ell_{t,i^\star}}^2 \right]
    \\
    &\leq
    \E \left[ \sum_{k \neq k^\star} \ptclique{t}{k}^{-\tfrac12} \tilde{p}_t(\clique_k) \right] + \E \left[ \ptclique{t}{k^\star}^{-\tfrac12} \brk3{\sum_{i \in \clique_{k^\star}, i \neq i^\star} \tilde{p}_{t,i} + \tilde{p}_{t,i^\star} \brk*{1 - \ptclique{t}{k^\star}}^2} \right] \\
    &\leq
    3 \E\left[ \sum_{k \neq k^\star} \sqrt{\ptclique{t}{k}} \right] + 2 \E \left[ \tilde{p}_t(\clique_{k^\star})^{-\tfrac12} \tilde{p}_t \brk*{\clique_{k^\star} \setminus i^\star} \right] + 3 \E \left[ \brk*{1 - \ptclique{t}{k^\star}}^2 \right] \\
    &\leq
    6 \E \left[ \sum_{k \neq k^\star} \sqrt{\ptclique{t}{k}} \right] + 2 \E \left[ \sqrt{\tilde{p}_t \brk*{\clique_{k^\star} \setminus i^\star}} \right] \\
    &\leq
    8 \E \left[ \sum_{k \neq k^\star} \sqrt{\ptclique{t}{k}} \right] + 2 \E \left[ \sqrt{\phalfcliqueminus{k^\star}{i^\star}} \right] \\
    &\leq
    8 \sum_{k \neq k^\star} \sqrt{\E[\ptclique{t}{k}]} + 2 \sqrt{\E \brk[s]*{\phalfcliqueminus{k^\star}{i^\star}}},
\end{align*}
where in the last inequality we used Jensen's inequality. We now proceed to bound the RHS of \cref{eqn:stability-other-cliques}:
\begin{align}
    \E\left[ \sum_{\clique_k \neq \clique(I_t)} \sqrt{\tilde{p}_t(\clique_k)} \sum_{i \in \clique_k} \tilde{p}_{t,i} (\ell_{t,i^\star})^2 \right]
    &\leq
    \E\left[ \sum_{\clique_k \neq \clique(I_t)} \tilde{p}_t(\clique_k)^{\frac32} \right] \nonumber\\
    &\leq
    6\E\left[ \sum_{\clique_k \neq \clique(I_t)} \ptclique{t}{k}^{\frac32} \right] \label{eqn:q-stab-neg}\\
    &=
    6\E\left[ \sum_{k=1}^K (1-\ptclique{t}{k}) \ptclique{t}{k}^{\frac32} \right] \label{eqn:prob-not-chosen}\\
    &\leq
    12\E\left[ \sum_{k \neq k^\star} \sqrt{\ptclique{t}{k}} \right] \nonumber\\
    &\leq 
    12 \sum_{k \neq k^\star} \sqrt{\E[\ptclique{t}{k}]}, \nonumber
\end{align}
where in \cref{eqn:q-stab-neg} we use \cref{lem:q-stability-new} and the fact that $\tilde{p}_t$ lies between $p_t$ and $\phalf$, in \cref{eqn:prob-not-chosen} we use the fact that the probability of the clique $\clique_k$ not to be chosen at time step $t$ is $1 - \ptclique{t}{k}$ and the last line uses Jensen's inequality. Combining the two bounds, we conclude the proof.
\end{proof}

\subsection[Proof of Penalty Bound]{Proof of \cref{lem:penalty-main} (Penalty)}
\label{sec:penalty-proofs}

In this section we restate \cref{lem:penalty-main} which bounds the penalty term to include the extra constants and poly-log factors.

\begin{manualrestate}{\cref{lem:penalty-main}}[restated]
\label{lem:penalty}
The penalty term described in the RHS of \cref{eqn:ftrl-penalty} is bounded by 
\begin{align}
    9K \log \frac{1}{\gamma}
    +
    5 \log^2 \frac{1}{\gamma} \sum_{t=1}^T \sum_{k \neq k^\star} \sqrt{\frac{\ptclique{t}{k}}{t}}
    +
    2 \log \frac{1}{\gamma} \sum_{t=1}^T \sqrt{\frac{p_t \brk*{\clique_{k^\star} \setminus i^\star}}{t}}, \label{eqn:penalty}
\end{align}
where
$
    p^\gamma_i = 
    \begin{cases}
    \gamma & i \neq i^\star \\
    1 - (N-1) \gamma & i = i^\star
    \end{cases}
$
for all $i\in [N]$ and $\frac{1}{\eta_0} \eqdef 0$.
\end{manualrestate}

\begin{proof}
Noting that $\Phi(\cdot) \geq 0$ we can bound the first term as follows:
\begin{align*}
    \Phi(p^\gamma) - \Phi(p_1)
    &\leq
    \Phi(p^\gamma) 
    \leq 
    9K \log \frac{1}{\gamma}.
\end{align*}
Continuing with the second part of the penalty term, note that by definition of $p^\gamma$ we have $p^\gamma(\clique_{k^\star}) \geq \ptclique{t}{k^\star}$ for all $t$. Also note that $\Psi(p^\gamma) \leq -2 \brk*{\log^2 \frac{1}{\gamma} + 1}\sqrt{p^\gamma(\clique_{k^\star})}$. We then have
\begin{align}
    \sum_{t=1}^T \brk*{\frac{1}{\eta_t} - \frac{1}{\eta_{t-1}}} \brk*{\Psi(p^\gamma) - \Psi(p_t)}
    &\leq
    \sum_{t=1}^T \brk*{\sqrt{t} - \sqrt{t-1}} \Bigg(2 \brk*{\log^2 \frac{1}{\gamma} + 1} \sum_{k=1}^K \sqrt{\ptclique{t}{k}} \nonumber
    \\
    &+ \sum_{k=1}^K \frac{1}{\sqrt{\ptclique{t}{k}}} \sum_{i \in \clique_k} p_{t,i} \log \frac{\ptclique{t}{k}}{p_{t,i}} - 2 \brk*{\log^2 \frac{1}{\gamma} + 1} \sqrt{p^\gamma(\clique_{k^\star})}\Bigg) \nonumber
    \\
    &\leq
    2\brk*{\log^2 \frac{1}{\gamma} + 1} \sum_{t=1}^T \frac{1}{\sqrt{t}} \sum_{k \neq k^\star} \sqrt{\ptclique{t}{k}} \nonumber \\
    &+ 
    \log \frac{1}{\gamma} \sum_{t=1}^T \frac{1}{\sqrt{t}} \sum_{k \neq k^\star} \sqrt{\ptclique{t}{k}} \nonumber\\
    &+
    \sum_{t=1}^T \frac{1}{\sqrt{t \cdot \ptclique{t}{k^\star}}} \sum_{i \in \clique_{k^\star}} p_{t,i} \log \frac{\ptclique{t}{k^\star}}{p_{t,i}}
    \nonumber\\
    &\leq
    5 \log^2 \frac{1}{\gamma} \sum_{t=1}^T \frac{1}{\sqrt{t}} \sum_{k \neq k^\star} \sqrt{\ptclique{t}{k}} \nonumber\\
    &+
    \sum_{t=1}^T \frac{1}{\sqrt{t \cdot \ptclique{t}{k^\star}}} \sum_{i \in \clique_{k^\star}} p_{t,i} \log \frac{\ptclique{t}{k^\star}}{p_{t,i}}, \label{eqn:penalty-mid-term}
\end{align}
where the second inequality follows from the fact that $\sqrt{t} - \sqrt{t-1} \leq \frac{1}{\sqrt{t}}$ and that $p_{t,i} \geq \gamma$ for all $t$ and $i$.
It is left to bound the final term. Using the inequality $\log x \leq x-1$ for all $x>0$ we have
\begin{align*}
    \sum_{i \in \clique_{k^\star}} p_{t,i} \log \frac{\ptclique{t}{k^\star}}{p_{t,i}}
    &=
    \sum_{i \in \clique_{k^\star} \setminus i^\star} p_{t,i} \log \frac{\ptclique{t}{k^\star}}{p_{t,i}}
    +
    p_{t,i^\star} \log \frac{\ptclique{t}{k^\star}}{p_{t,i^\star}} \\
    &\leq
    \log \frac{1}{\gamma} \sum_{i \in \clique_{k^\star} \setminus i^\star} p_{t,i}
    +
    p_{t,i^\star} \brk*{\frac{\ptclique{t}{k^\star}}{p_{t,i^\star}} - 1} \\
    &=
    \brk*{\log \frac{1}{\gamma} + 1} p_t \brk*{\clique_{k^\star} \setminus i^\star} \\
    &\leq
    2 \log \frac{1}{\gamma} p_t \brk*{\clique_{k^\star} \setminus i^\star}.
\end{align*}
Plugging this bound into \cref{eqn:penalty-mid-term} while using the fact that $\frac{p_t \brk*{\clique_{k^\star} \setminus i^\star}}{\sqrt{\ptclique{t}{k^\star}}} \leq \sqrt{p_t \brk*{\clique_{k^\star} \setminus i^\star}}$ completes the proof.
\end{proof}

\subsection[Proof of Main Theorem]{Proof of \cref{thm:graph-regret-bound-main}}
\label{sec:graph-regret-bound-proof}

In order to prove \cref{thm:graph-regret-bound-main} we make use of the following simple claim which asserts that the pseudo-regret is bounded up to an additive constant factor by the regret with respect to some probability vector in $\simplexftrl$.

\begin{lemma}
\label{lem:shrinked-regret}
For all $\gamma \in [0,\frac{1}{N}]$ and $i^\star \in [N]$ the following holds:
\begin{align*}
    \E \left[\sum_{t=1}^T p_t \cdot \estell_t - \mathbf{e}_{i^\star} \cdot \sum_{t=1}^T \estell_t  \right] \leq \mathbb{E} \left[ \sum_{t=1}^T p_t \cdot \estell_t - p^\gamma \cdot \sum_{t=1}^T \estell_t \right] + \gamma T N,
\end{align*}
where 
$
p^\gamma_i = 
    \begin{cases}
    \gamma & i \neq i^\star \\
    1 - (N-1) \gamma & i = i^\star
    \end{cases} 
    \quad \forall i \in [N].
$
\end{lemma}

\begin{proof}
Fix $\gamma \in [0, \frac{1}{N}]$ and $i^\star \in [N]$. Note that $\mathbf{e}_{i^\star} = p^\gamma - v$ where $v$ is defined as follows:
\begin{align*}
    v_i = 
    \begin{cases}
    \gamma & i \neq i^\star \\
    -(N-1) \gamma & i = i^\star
    \end{cases} 
    \quad \forall i \in [N].
\end{align*}
This observation gives us the following:
\begin{align*}
    \E \left[\sum_{t=1}^T  p_t \cdot \estell_t - \mathbf{e}_{i^\star} \cdot \sum_{t=1}^T \estell_t  \right]
    &=
    \E \left[\sum_{t=1}^T  p_t \cdot \ell_t - \mathbf{e}_{i^\star} \cdot \sum_{t=1}^T \ell_t  \right] \\
    &=
    \E \left[\sum_{t=1}^T  p_t \cdot \ell_t - p^\gamma \cdot \sum_{t=1}^T \ell_t  \right] + v \cdot \E \left[ \sum_{t=1}^T \ell_t \right] \\
    &=
    \E \left[\sum_{t=1}^T  p_t \cdot \estell_t - p^\gamma \cdot \sum_{t=1}^T \estell_t  \right] + v \cdot \E \left[ \sum_{t=1}^T \ell_t \right],
\end{align*}
where the first equality is due to the fact that $\estell_t$ is an unbiased estimator of $\ell_t$. We bound the last term using the expression for $v$:
\begin{align*}
    v \cdot  \sum_{t=1}^T \ell_t
    =
    \sum_{t=1}^T \left[ \sum_{i \neq i^\star} \gamma \ell_{t,i} - (N-1) \gamma \ell_{t,i^\star} \right]
    \leq 
    \gamma T N,
\end{align*}
where in the last inequality we use the fact that the losses are bounded in $[0,1]$.
\end{proof}

We will also make use of general FTRL regret bound given by \cref{thm:ftrl-general-regret-bound-haipeng} (which we prove in \cref{sec:general-regret-bound}) together with the stability and penalty bounds shown in the previous sections. 
\cref{thm:graph-regret-bound-main} is restated here in the precise form proved below.

\begin{manualrestate}{\cref{thm:graph-regret-bound-main}}[restated]
\label{thm:graph-regret-bound}
\cref{alg:FTRL} attains the following regret bound, regardless of the corruption level, for $NT \geq 3^{11}$:
\begin{align}
\label{eqn:graph-reg-bound-factors}
    \regret
    &\leq 
    9K \log(NT) 
    + 6 \log^2(NT)\sum_{t=1}^T \sumcliqueneq \sqrt{ \frac{\E\brk[s]*{\ptclique{t}{k}}}{t}} \nonumber \\
    &+ 
    2 \log(NT) \sum_{t=1}^T \sqrt{\frac{\E[p_t(\clique_{k^\star} \setminus i^\star)]}{t}}
    + 16 \sum_{t=1}^T \sqrt{\frac{\E[\phalfcliqueminus{k^\star}{i^\star}]}{t}}
    .
\end{align}
\end{manualrestate}

\begin{proof}
Note that due to \cref{lem:shrinked-regret} it suffices to bound $\E \brk[s]*{\sum_{t=1}^T \brk*{p_t - p^\gamma} \cdot \estell_t}$ where $p^\gamma$ is defined by
\begin{align*}
    p^\gamma_i = 
    \begin{cases}
    \gamma & i \neq i^\star \\
    1 - (N-1) \gamma & i = i^\star
    \end{cases}
    \quad
    \forall i \in [N],
\end{align*}
since it can only be larger than the pseudo-regret by an additive constant.
Using \cref{thm:ftrl-general-regret-bound-haipeng} and then bounding the penalty and stability terms using \cref{lem:penalty} and \cref{lem:stability-term-bound-full} we obtain
\begin{align*}
    \regret
    &\leq
    9K \log(NT) 
    + 
    \brk*{5 \log^2(NT) + 112}\sum_{t=1}^T \sum_{k \neq k^\star} \sqrt{\frac{\E[\ptclique{t}{k}]}{t}} \\ 
    &+
    2 \log (NT) \sum_{t=1}^T \sqrt{\frac{\E[p_t(\clique_{k^\star} \setminus i^\star)]}{t}}
    +
    16 \sum_{t=1}^T \sqrt{\frac{\E[\phalfcliqueminus{k^\star}{i^\star}]}{t}} \\
    &\leq
    9K \log(NT) + 6 \log^2(NT)\sum_{t=1}^T \sum_{k \neq k^\star} \sqrt{\frac{\E[\ptclique{t}{k}]}{t}} \\
    &+
    2 \log (NT) \sum_{t=1}^T \sqrt{\frac{\E[p_t(\clique_{k^\star} \setminus i^\star)]}{t}}
    +
    16 \sum_{t=1}^T \sqrt{\frac{\E[\phalfcliqueminus{k^\star}{i^\star}]}{t}}
    ,
\end{align*}
where the last inequality holds for $NT \geq 3^{11}$.
\end{proof}

\subsection[Proof of Main Theorem (Full)]{Proof of \cref{thm:reg-corrupted} (Main)}
\label{sec:corrupted-bound-proof-full}

We can now provide a proof of our main result given in \cref{thm:reg-corrupted}, restated here more precisely.

\begin{manualrestate}{\cref{thm:reg-corrupted}}[restated]
\label{thm:reg-corrupted-full}
\cref{alg:FTRL} attains the following expected pseudo-regret bound in the $C$-corrupted stochastic setting, for $NT \geq 3^{11}$:
\begin{align*}
    \regret
    &\leq
    184\log^2(NT) \cdot \min \brk[c]*{\sqrt{KT}, \log^2(NT) \deltasum{\log T} + \sqrt{C \deltasum{\log T}}}.
\end{align*}
\end{manualrestate}

\begin{proof}
We first prove the following:
\begin{align*}
    \regret
    \leq
    184 \log^4(NT) \deltasum{\log T} + 28 \log^2(NT) \sqrt{C \deltasum{\log T}}.
\end{align*}
We proceed bounding the RHS of \cref{eqn:graph-reg-bound-factors}. For all $B,z > 0$ we have
\begin{align}
    B \sum_{t=1}^T \brk*{ \sum_{k \neq k^\star} \sqrt{\frac{\E[\ptclique{t}{k}]}{t}}
    + \sqrt{\frac{\E[p_t(\clique_{k^\star} \setminus i^\star)]}{t}}
    }
    &\leq
    B^2 \cdot z \sum_{t=1}^T \sum_{k : \Delta_k > 0} \frac{1}{2t \Delta_k} + \frac{1}{2z} \sum_{t=1}^T \sum_{i=1}^N \E[p_{t,i}] \Deltaarm{i} \nonumber\\
    &\leq
    B^2 \cdot z \deltasum{\log T} + \frac{1}{2z} \sum_{t=1}^T \sum_{i=1}^N \E[p_{t,i}] \Deltaarm{i} \nonumber\\
    &\leq
    B^2 \cdot z \deltasum{\log T} + \frac{1}{2z} \brk*{\regret + 2C}, \label{eqn:cor-bound-1}
\end{align}
where the first inequality is due to Young's inequality and the fact that $\Deltaclique{k} \leq \Deltaarm{i}$ for all $i \in \clique_k$, the second inequality is since $\sum_{t=1}^T (\ifrac{1}{t}) \leq 2 \log T$ and the last inequality is due to the following simple observation which follows from the definition of corruption:
\begin{align*}
    \E \brk[s]*{\sum_{t=1}^T \sum_{i=1}^N 
    p_{t,i} \brk*{\iidell_{t,i} - \iidell_{t,i^\star}}}
    &\leq
    \E \brk[s]*{\sum_{t=1}^T \sum_{i=1}^N 
    p_{t,i} \brk*{\ell_{t,i} - \ell_{t,i^\star}}} 
    +
    2 \E \brk[s]*{\sum_{t=1}^T \norm{\corell_t - \iidell_t}_{\infty}} \\
    &=
     \E \brk[s]*{\sum_{t=1}^T \sum_{i=1}^N 
    p_{t,i} \brk*{\ell_{t,i} - \ell_{t,i^\star}}} 
    +
    2C.
\end{align*}
Setting $B = 6\log^2(NT)$ gives a bound on the second term in the RHS of \cref{eqn:graph-reg-bound-factors}.
Similarly, we have
\begin{align}
    16 \sum_{t=1}^T \sqrt{\frac{1}{t} \E[\phalfcliqueminus{k^\star}{i^\star}]}
    &\leq
    256z \deltasum{\log T} + \frac{1}{2z} \sum_{t=1}^T \sum_{i=1}^N \E[\phalfi{i}] \Deltaarm{i} \nonumber\\
    &\leq
    256z \deltasum{\log T} + \frac{C}{z} + \frac{1}{2z} \E \brk[s]*{\sum_{t=1}^T \sum_{i=1}^N \phalfi{i} \cdot \brk*{\corell_{t,i} - \corell_{t,i^\star}}}.
    \label{eqn:phalf-regret}
\end{align}
We now use \cref{lem:ftrl-shifted-regret} to bound the rightmost term of \cref{eqn:phalf-regret} as follows:
\begin{align*}
    \E \brk[s]*{\sum_{t=1}^T \sum_{i=1}^N \phalfi{i} \cdot \brk*{\corell_{t,i} - \corell_{t,i^\star}}}
    &=
    \E \brk[s]*{\sum_{t=1}^T \phalf \cdot \brk*{\E_t[\estell_t] - \corell_{t,i^\star} \mathbf{1}}} \\
    &\leq
    \E \brk[s]*{\sum_{t=1}^T p_t \cdot \brk*{\E_t[\estell_t] - \corell_{t,i^\star} \mathbf{1}}} \\
    &=
    \E \brk[s]*{\sum_{t=1}^T \sum_{i=1}^N p_{t,i} \cdot \brk*{\corell_{t,i} - \corell_{t,i^\star}}} \\
    &=
    \regret,
\end{align*}
where we used the fact that $\estell_t$ is an unbiased estimator for $\ell_t$. We can conclude that 
\begin{align}
     16 \sum_{t=1}^T \sum_{k : \Deltaclique{k} > 0} \sqrt{\frac{1}{t} \E[\phalfclique{k}]}
     &\leq
     256 \deltasum{\log T} + \frac{1}{2z} \brk*{\regret + 2C}. \label{eqn:cor-bound-2}
\end{align}
Using \cref{thm:graph-regret-bound} and combining the bounds from \cref{eqn:cor-bound-1} and \cref{eqn:cor-bound-2} we obtain
\begin{align*}
    \regret
    &\leq
    9K \log(NT) + \brk*{36 \log^4(NT) + 256} z \deltasum{\log T} + \frac{1}{z} \regret + \frac{2C}{z} \\
    &\leq
    \brk*{45 \log^4(NT) + 256} z \deltasum{\log T} + \frac{1}{z} \regret + \frac{2C}{z} \\
    &\leq
    46 \log^4(NT) z \deltasum{\log T} + \frac{1}{z} \regret + \frac{2C}{z},
\end{align*}
where the second inequality is since $K \leq 1 + \deltasum{1}$ and the last inequality holds since $NT \geq 3^4$. Rearranging and simplifying we obtain
\begin{align*}
    \regret
    &\leq
    2U + (z-1)U + \frac{2C + U}{z-1},
\end{align*}
where we denote $U = 46 \log^4(NT) \deltasum{\log T}$ for simplicity. We now choose $z$ which minimizes the bound, by setting $z = 1 + \sqrt{\frac{U+2C}{U}}$. This gives us
\begin{align*}
    \regret
    &\leq
    2U + 2\sqrt{U(U+2C)} \\
    &\leq
    4U + 4 \sqrt{UC} \\
    &\leq
    184 \log^4(NT) \deltasum{\log T} + 28 \log^2(NT) \sqrt{C \deltasum{ \log T}},
\end{align*}
which concludes the first part of the proof. We now show that
\begin{align*}
    \regret
    &\leq
    28 \log^2(NT) \sqrt{KT}.
\end{align*}
We again use \cref{thm:graph-regret-bound} and also the fact that $\phalfclique{k} \leq \frac73 \ptclique{t}{k}$ by \cref{lem:q-stability-new}, to obtain
\begin{align*}
    \regret 
    &\leq
    9K \log(NT) + \brk*{6 \log^2 (NT) + 32}\sum_{t=1}^T \frac{1}{\sqrt{t}} \sum_{k=1}^K \sqrt{\ptclique{t}{k}} \\
    &\leq
    9K \log(NT) + 7 \log^2(NT)\sum_{t=1}^T \frac{1}{\sqrt{t}} \sum_{k=1}^K \sqrt{\ptclique{t}{k}},
\end{align*}
where the inequality holds since $NT \geq 3^6$. We conclude the proof via the following straightforward calculation:
\begin{align*}
    \sum_{t=1}^T \frac{1}{\sqrt{t}} \sum_{k=1}^K \sqrt{\ptclique{t}{k}}
    \leq
    \sqrt{K} \sum_{t=1}^T \frac{1}{\sqrt{t}}
    \leq
    2\sqrt{KT},
\end{align*}
where we used Jensen's inequality and the fact that $\sum_{t=1}^T (\ifrac{1}{\sqrt{t}}) \leq 2\sqrt{T}$. We obtained two regret bounds and thus the minimum of the two holds, which concludes the proof.
\end{proof}

\section{Refined Regret Bound for FTRL}
\label{sec:general-regret-bound}

Consider the FTRL framework which generates predictions $w_1,w_2,...,w_T \in \W$ given a sequence of arbitrary loss vectors $g_1,g_2,...,g_T$ and a sequence of regularization functions $H_1,H_2,...,H_T$. The following gives a general regret bound which we use in order to prove \cref{thm:graph-regret-bound}.

\begin{theorem} 
\label{thm:ftrl-general-regret-bound-haipeng}
Suppose $H_t = \eta_t^{-1} \psi + \phi$ for twice-differentiable and convex functions $\psi$ and $\phi$, $\psi$ being strictly convex. Let $\whalf = \argmin_{w \in \W} \brk[c]!{w \cdot \sum_{s=1}^t g_s + H_t(w)}$. Then there exists a sequence of points $\tilde{w}_t \in [w_t, \whalf]$ such that, for all $w^* \in \W$:
\begin{align*}
    \sum_{t=1}^T g_t \cdot (w_t - w^\star)
    &\leq
    \phi(w^\star) - \phi(w_1) + \sum_{t=1}^T \brk*{\frac{1}{\eta_t} - \frac{1}{\eta_{t-1}}}(\psi(w^\star) - \psi(w_t)) + 2 \sum_{t=1}^T \eta_t \brk!{\norm{g_t}_t^*}^2.
\end{align*}
Here $\norm{g}_t = \sqrt{g \tr \grad ^2 \psi(\tilde{w}_t) g}$ is the local norm induced by $\psi$ at $\tilde{w}_t$, and $\norm{\cdot}_t^*$ is its dual. Here we also define $\ifrac{1}{\eta_0} \eqdef 0$.
\end{theorem}

\begin{proof}
We directly follow an analysis by \citet{jin2020simultaneously}, and include the details for completeness. For simplicity we denote $G_t = \sum_{s=1}^t g_s$. We make the following definitions:
\begin{align*}
    F_t(w)
    &=
    w \cdot G_{t-1} + H_t(w), \\
   \Fhalf(w) &= w \cdot G_t + H_t(w),
\end{align*}
such that $w_t = \argmin_{w \in \W} \brk[c]*{F_t(w)}$ and $\whalf = \argmin_{w \in \W} \brk[c]*{\Fhalf(w)}$. Fix $w^\star \in \W$. We note that the regret of FTRL with respect to $w^\star$ has the following decomposition:
\begin{align*}
    \sum_{t=1}^T g_t \cdot \brk*{w_t - w^\star}
    &=
    \sum_{t=1}^T \brk*{w_t \cdot g_t + F_t(w_t) - \Fhalf(\whalf)}
    +
    \sum_{t=1}^T \brk*{\Fhalf(\whalf) - F_t(w_t) - w^\star \cdot g_t}.
\end{align*}
We first show that for all time steps $t$ it holds that
\begin{align}
\label{eqn:general-stability-bound}
    w_t \cdot g_t + F_t(w_t) - \Fhalf(\whalf)
    &\leq
    2 \eta_t \brk*{\norm{g_t}_t^*}^2.
\end{align}
We lower bound $w_t \cdot g_t + F_t(w_t) - \Fhalf(\whalf)$ as follows:
\begin{align*}
    w_t \cdot g_t + F_t(w_t) - \Fhalf(\whalf)
    &=
    w_t \cdot G_t + H_t(w_t) - \Fhalf(\whalf) \\
    &=
    \Fhalf(w_t) - \Fhalf(\whalf) \\
    &=
    \grad \Fhalf(\whalf) \cdot \brk*{w_t - \whalf} + \ifrac12 \norm{w_t - \whalf}^2_{\grad ^2 H_t(\tilde{w}_t)} \\
    &\geq
    \ifrac12 \norm{w_t - \whalf}^2_{\grad ^2 H_t(\tilde{w}_t)} \\
    &\geq
    \ifrac12 \eta_t^{-1} \norm{w_t - \whalf}^2_t,
\end{align*}
where the third line is a Taylor expansion of $\Fhalf$ around $\whalf$, with $\tilde{w}_t$ being a point between $w_t$ and $\whalf$, in the second to last line we use a first-order optimality condition of $\whalf$, and in the last line we use the fact that $\grad^2 H_t \succeq \eta_t^{-1} \grad^2 \psi$. We now upper bound $w_t \cdot g_t + F_t(w_t) - \Fhalf(\whalf)$ as follows:
\begin{align*}
    w_t \cdot g_t + F_t(w_t) - \Fhalf(\whalf)
    &=
    \brk*{w_t - \whalf} \cdot g_t + F_t(w_t) - F_t(\whalf) \\
    &\leq
    \brk*{w_t - \whalf} \cdot g_t \\
    &\leq
    \brk*{\sqrt{\eta_t^{-1}} \norm{w_t - \whalf}_t} \brk*{\sqrt{\eta_t} \norm{g_t}_t^*} \\
    &=
    \norm{w_t - \whalf}_t \cdot \norm{g_t}_t^*,
\end{align*}
where in the first inequality we use the fact that $w_t$ is the minimizer of $F_t$ and the second inequality is an application of H\"older's inequality. Combining the lower and upper bounds gives us \cref{eqn:general-stability-bound}. Next we show that
\begin{align}
\label{eqn:general-penalty-bound}
    \sum_{t=1}^T \brk*{\Fhalf(\whalf) - F_t(w_t) - w^\star \cdot g_t}
    &\leq
    \phi(w^\star) - \phi(w_1) + \sum_{t=1}^T \brk*{\frac{1}{\eta_t} - \frac{1}{\eta_{t-1}}}(\psi(w^\star) - \psi(w_t)).
\end{align}
We bound the LHS of \cref{eqn:general-penalty-bound} as follows:
\begin{align*}
    &\sum_{t=1}^T \brk*{\Fhalf(\whalf) - F_t(w_t) - w^\star \cdot g_t} \\
    &\leq 
    -F_1(w_1) + \sum_{t=2}^T \brk*{F_{t-1}^+(w_t) - F_t(w_t)} + F_T^+(w_T^+) - w^\star \cdot G_T \\
    &\leq
    -F_1(w_1) + \sum_{t=2}^T \brk*{F_{t-1}^+(w_t) - F_t(w_t)} + F_T^+(w^\star) - w^\star \cdot G_T \\
    &=
    -H_1(w_1) - \sum_{t=2}^T \brk*{ \frac{1}{\eta_t} - \frac{1}{\eta_{t-1}}} \psi(w_t) + H_T(w^\star) \\
    &=
    -\eta_1^{-1} \psi(w_1) - \phi(w_1) - \sum_{t=2}^T \brk*{ \frac{1}{\eta_t} - \frac{1}{\eta_{t-1}}} \psi(w_t) + \eta_T^{-1} \psi(w^\star) + \phi(w^\star) \\
    &=
    \phi(w^\star) - \phi(w_1) + \sum_{t=1}^T \brk*{\frac{1}{\eta_t} - \frac{1}{\eta_{t-1}}} \brk*{\psi(w^\star) - \psi(w_t)},
\end{align*}
where in the first and second inequalities we use the optimality of $\whalf$. Combining \cref{eqn:general-stability-bound} and \cref{eqn:general-penalty-bound} we conclude the proof.
\end{proof}

\begin{proof} [Proof of \cref{lem:general-regret-bound}]
Fix any $p^\gamma \in \simplexftrl$. The lemma follows immediately by applying \cref{thm:ftrl-general-regret-bound-haipeng} to \cref{alg:FTRL} with the regularizations $R_1,R_2,...,R_T$ and the shifted loss estimators $\estell_t - \ell_{t,i^\star} \mathbf{1}$, while noting that constant shifts in the loss estimators do not change the algorithm whatsoever.
\end{proof}